\documentclass[times,twocolumn,final]{elsarticle}

\usepackage{ycviu}

\usepackage{amsmath,amsfonts,amssymb,mathrsfs}

\usepackage{url}
\usepackage{caption}
\usepackage{lscape} 
\usepackage{rotating}
\usepackage{acronym} 
\usepackage{threeparttable}
\usepackage{soul}
\usepackage{subcaption}
\usepackage{multirow}
\usepackage{float}
\usepackage{hyperref}
\usepackage[table]{xcolor}
\usepackage{makecell, cellspace, caption}
\usepackage{array}
\usepackage{diagbox}
\usepackage{tikz}
\usepackage{booktabs}
\usepackage{breqn}
\usepackage{mathrsfs}
\usepackage{amsthm}
\usepackage{svg}
\usepackage{hhline}
\usepackage{colortbl}
\newtheorem{theorem}{Theorem}

\newtheorem{lemma}{Lemma}

\usepackage[linesnumbered,ruled,vlined]{algorithm2e}

\SetCommentSty{mycommfont}

\definecolor{Gainsboro}{rgb}{0.86, 0.86, 0.86}

\DeclareUnicodeCharacter{2212}{\ensuremath{-}}
\usepackage{arydshln}
\definecolor{newcolor}{rgb}{.8,.349,.1}
\newcolumntype{M}{@{\extracolsep{0.15cm}}>{\color{black}}c@{\extracolsep{0pt}}}%
\usepackage[]{xcolor}
\journal{Computer Vision and Image Understanding}

\begin{document}

\begin{frontmatter}
 

%
\title{MF-GCN: A Multi-Frequency Graph Convolutional Network for Tri-Modal Depression Detection Using Eye-Tracking, Facial, and Acoustic Features}

\author[1]{Sejuti Rahman\corref{cor1}}
\ead{s.rahman@newuu.uz}
\author[3]{Swakshar Deb}
\ead{swd9tc@virginia.edu}
\author[2]{MD. Sameer Iqbal Chowdhury}
\ead{sameer.iqbal2171@gmail.com}
\author[2]{MD. Jubair Ahmed Sourov}
\ead{sourovjubair@du.ac.bd}
\author[2]{Mohammad Shamsuddin}
\ead{mdshams013@gmail.com}

\cortext[cor1]{Corresponding author.}

\address[1]{Department of Computer Science, New Uzbekistan University, Tashkent, Uzbekistan}
\address[2]{Department of Robotics and Mechatronics Engineering, University of Dhaka, Bangladesh}
\address[3]{Department of Electrical and Computer Engineering, University of Virginia, USA}


\end{frontmatter}

\section{Introduction}
Depression, a major mental health disorder, refers to persistent feelings of sadness, hopelessness, and loss of interest or pleasure in activities, typically lasting for at least two weeks \cite{apa2013diagnostic}. This condition represents one of the most common mental health challenges experienced by individuals worldwide, affecting approximately 5\% of adults globally according to the World Health Organization \cite{who2021depression}. However, comprehensive reviews spanning multiple countries reveal considerable regional and demographic variations, with depressive symptoms affecting between 4.4\% and 20\% of the general population \cite{ferrari2013burden}.\par
The impact of depression extends far beyond individual suffering, with significant consequences for personal well-being and social functioning. At the individual level, those with depression often exhibit decreased productivity, compromised overall health, and increased susceptibility to various physical illnesses \cite{evans2005impact}. Additionally, individuals with depression may become vulnerable to developing other mental health disorders and physical health complications \cite{moussavi2007depression}. The consequences ripple outward to create a substantial public health concern, impacting not only affected individuals, but also their families, workplaces, and communities \cite{kessler2012effects}.
In its most severe manifestations, depression may lead to self-harm and suicide \cite{hawton2013risk}. Despite these serious implications, depression is frequently underestimated as a health concern \cite{simon2002outpatient}. This underestimation underscores the critical importance of prompt intervention and support for those affected and their broader communities.

Depression manifests in various forms, each with its own unique characteristics and challenges. Major Depressive Disorder (MDD) is perhaps the most widely recognized type, characterized by persistent feelings of sadness, hopelessness, and loss of interest in activities that once brought joy \cite{apa2013diagnostic}. 

MDD significantly impairs daily functioning, and disrupts sleep, appetite, and energy levels. Without treatment, episodes can last for months and often recur throughout life. Persistent Depressive Disorder (dysthymia), a chronic but typically less severe form of depression, involves persistent symptoms such as a consistently low mood, low self-esteem, poor concentration, and hopelessness, lasting at least two years in adults\cite{klein2010chronic}.

Apart from MDD and Dysthymia, clinical depressive disorders also include Bipolar Disorder, formerly known as manic-depressive illness, which is characterized by alternating periods of depression and mania or hypomania, with depressive episodes resembling those of MDD, but interspersed with manic or hypomanic episodes marked by elevated mood, increased energy, and sometimes reckless behavior \cite{goodwin2007manic}. Seasonal Affective Disorder (SAD) is a type of depression following a seasonal pattern, typically occurring in fall and winter, with symptoms such as low energy, oversleeping, and social withdrawal, highlighting the interplay between biological rhythms and environmental factors \cite{rosenthal1984seasonal}. Postpartum Depression, which affects women within the first year after childbirth, involves symptoms such as intense sadness, fatigue, and bonding difficulties driven by hormonal changes and the stress of caring for a newborn \cite{obrien2017postpartum}. Psychotic Depression, a severe form of depression accompanied by psychotic symptoms like delusions or hallucinations, often with themes of guilt or punishment, requires a combination of antidepressant and antipsychotic treatments due to its complexity and severity \cite{rothschild2013challenges}.

Understanding the various types of depression is crucial for effective diagnosis and treatment. Although they share some common features, each type has unique characteristics, risk factors, and recommended treatment approaches. It is important to acknowledge that individuals may not always fit neatly into one category, and some may exhibit symptoms that overlap with multiple types of depression \cite{Benazzi30062006}. Moreover, cultural factors, personal experiences, and individual differences can shape how depression manifests and is experienced. 

Given this intricate landscape of depressive disorders, our study sought to address their shared and divergent features by focusing on a key starting point: MDD. In our work, we placed particular emphasis on it, given its prominence as the most common form of depression and its frequent overlap with the symptomatology of other depressive disorders. By prioritizing the accurate diagnosis of MDD, we established a foundational framework that can streamline the identification of this condition and, with additional effort, facilitate the recognition of other types of depression in future investigations. As research in this field continues to evolve, our understanding of these disorders and their optimal treatments continue to improve, offering hope for those affected by these challenging conditions. 

Currently, the clinical detection of depression often relies on standardized tools, such as the Diagnostic and Statistical Manual of Mental Disorders, Fifth Edition (DSM-5) \cite{DSM5_1980_book} and the Patient Health Questionnaire-9 (PHQ-9) \cite{kroenke_2010_PHQ9}. The DSM-5, developed by the American Psychiatric Association, provides diagnostic criteria based on comprehensive assessment of symptoms, duration, and functional impairment---requiring at least five of nine specified symptoms (including depressed mood or loss of interest) persisting for two weeks to diagnose MDD. The PHQ-9 is a brief, self-administered screening tool that quantifies depression severity through a nine-item questionnaire aligned with DSM-5 criteria, yielding scores from 0 to 27 where higher scores indicate greater symptom severity.

However, these methods are not without limitations: the DSM-5's heavy reliance on expert judgment introduces significant subjectivity, as diagnostic decisions depend on clinician interpretation of symptom severity, functional impairment, and complex differential diagnoses, making assessments vulnerable to inter-rater variability and practitioner experience levels. Additionally, the PHQ-9, though practical, faces several challenges including subjective bias where patients may underreport or overreport symptoms due to social desirability or denial.There is also temporal limitations as it captures only a two-week snapshot potentially missing longer-term patterns, language and literacy barriers that affect comprehension across diverse populations, and somatic symptom overlap where items about fatigue, sleep, and appetite may be elevated due to medical conditions unrelated to depression. The PHQ-9 may also miss nuanced presentations of depression or overestimate severity in patients with overlapping conditions, such as anxiety or chronic illness. Despite these challenges, both tools remain fundamental in guiding diagnosis and informing treatment, underscoring the need for continued refinement in clinical practice.


In response to these challenges, recent advances in depression diagnosis methodologies have emerged as alternatives to traditional clinical assessments, with the aim of mitigating the human bias inherent in psychological evaluations. Researchers are increasingly interested in aiding depression diagnosis utilizing data from various sources. Detection of depression using data-driven methods uses large datasets and machine learning (ML) to identify patterns of depression symptoms. Data-driven methods are superior to traditional diagnostic methods in terms of objectivity, scalability, and early identification. Data sources typically include voice and speech from call recordings, biometric data from wearable sensors (e.g., heart rate,  fMRI), behavioral data, and clinical records. ML models are trained on labeled datasets to find patterns to identify depressive symptoms early and objectively. Lin et al. \cite{lin2021depression} demonstrated the effectiveness of saliency-based diagnostic methods, achieving an accuracy rate of 90.01\%. However, relying solely on single metrics such as saliency may not provide a complete diagnostic picture.

Multimodal approaches for automatic depression detection have shown considerable promise in recent years, leveraging complementary information provided by different data sources to improve accuracy and robustness. These methods typically combine auditory and visual modalities, capitalizing on the observed differences in speech patterns and facial expressions between depressed and healthy individuals \cite{cai2020feature}. By fusing the features from multiple modalities, these approaches can capture a more comprehensive representation of depressive symptoms.
For example, Niu et al. proposed a novel spatiotemporal attention network that integrates spatial and temporal information from audio and video data, emphasizing the most relevant frames for depression detection \cite{niu2023multimodal}. Their multimodal attention feature fusion strategy demonstrated superior performance compared to unimodal methods on standard benchmarks. Similarly, Yang et al. \cite{yang2017multimodal} developed a multimodal deep learning framework that jointly learns features from audio, video, and text modalities, achieving state-of-the-art results on the AVEC2013 and AVEC2014 depression datasets.
Other researchers have explored the fusion of physiological signals with audio-visual data. For example, Katyal et al. \cite{katyal2014eeg} combined EEG signals with facial video features and showed that the addition of neurophysiological information can improve the accuracy of depression detection. 
Moreover, Dibeklioglu et al. \cite{dibeklioglu2018dynamic} integrated facial, head, and vocal prosody features into a multimodal framework, demonstrating improved performance over single-modality approaches.
The success of these multimodal methods can be attributed to their ability to capture diverse aspects of depressive behavior, from subtle changes in vocal intonation to microexpressions and body language. As sensing technologies continue to advance, integrating additional modalities such as smartphone usage patterns or social media activity may further enhance the capabilities of automatic depression detection systems \cite{mohr2017personal}.

While existing multimodal approaches have demonstrated the value of combining different data sources, the selection of specific modalities should be guided by their clinical relevance to MDD symptomatology.
In examining the diverse modalities through which depression manifests, our investigation identified eye tracking, facial expression, and audio as particularly salient for detecting MDD because of their robust correlations with MDD symptomatology. Eye tracking research demonstrates that individuals with MDD exhibit distinct patterns, such as prolonged fixation on negative stimuli and reduced attention to positive cues, signaling underlying attentional biases, and emotional dysregulation \cite{armstrong2010eye}. Analysis of facial expressions revealed that patients with MDD often show reduced expressivity, marked by fewer smiles and heightened displays of sadness or affective flattening, reflecting emotional withdrawal and blunted affect \cite{girard2014nonverbal}. Similarly, audio analysis uncovered vocal markers of MDD, including slower speech, extended pauses, and diminished pitch variation, which correspond to psychomotor slowing and persistent low mood \cite{cummins2015review}. Given these compelling associations with core MDD features, we selected these three modalities for our research, leveraging their potential to enhance diagnostic accuracy and provide a multimodal approach to understanding this disorder.

To effectively integrate information from these diverse modalities and capture their complex inter-relationships, we employ graph-based computational methods. Recent graph-based methods \cite{xia2024_multimodal_depression} 
that model relationships between different data modalities have shown promise in capturing complex interactions between behavioral markers. However, many existing approaches focus primarily on low-frequency information, potentially overlooking critical diagnostic patterns in high-frequency features. Our work addresses this limitation by developing MF-GCN (Multi-Frequency Graph Convolutional Network), a novel graph-based framework that integrates our Multi-Frequency Filter-Bank Module (MFFBM). Unlike conventional graph convolutional networks that rely on fixed low-frequency filtering, MFFBM enables learning from both low- and high-frequency spectral information, allowing for more comprehensive analysis of multi-modal depression data.


The contributions of this study are manifold:

\begin{itemize}
    \item We developed a gold standard depression dataset comprising 103 participants across diverse age groups (17-56 years) with ground truth PHQ-9 values provided by multiple psychiatrists. Unlike existing datasets (such as DAIC-WOZ \cite{gratch2014DAIC} and AVEC \cite{valstar2013avec, valstar2014avec} 
    that only collect audio and video signals, our dataset uniquely incorporates eye-tracking data that capture patients' gaze patterns, which we demonstrate as indicative of depression severity. This trimodal approach (audio, video, and eye tracking) creates a more comprehensive resource for depression detection research, enabling a more robust multimodal analysis.
    
    \item While Lin et al.~\cite{lin2021depression} pioneered saliency-based depression detection as a unimodal approach, our work is the first to systematically evaluate its effectiveness in multimodal frameworks. We demonstrate that integrating saliency maps with audio and facial video data enables significant performance improvements (a 19\% increase in F2-score compared to unimodal approach) by capturing complementary information across modalities and enhancing the detection of subtle depression indicators that single-modality approaches might miss.
    
    \item We propose MF-GCN (Multi-Frequency Graph Convolutional Network) with our novel Multi-Frequency Filter-Bank Module (MFFBM) to address a critical limitation in existing graph-based depression detection models that focus primarily on low-frequency information. Through empirical validation, we demonstrate that high-frequency features contain critical diagnostic information that current approaches overlook. We provide theoretical analysis with mathematical proof that MFFBM enables learning arbitrary spectral filters, transcending the fixed filtering capabilities of conventional graph convolutional networks. Comprehensive evaluation against seven traditional machine learning methods and three state-of-the-art multimodal deep learning approaches demonstrates that MF-GCN achieves 96\% sensitivity and 0.94 F2-score in binary depression classification, consistently outperforming all baselines across both binary and three-class classification tasks.
\end{itemize}

\section{Related works}

We sifted through and studied relevant scholarly works on computational approaches used to detect depression. We organized our review by grouping the various works by the modalities they explore: eye tracking, facial expression, audio, and those that employ multimodal approaches to incorporate them.

\subsection{Depression detection from eye tracking modality}

Visual attention patterns and eye movements can provide significant insight into mental health conditions, particularly depression. Research has shown that emotional states significantly influence how individuals view and process visual information, particularly when they are engaged in emotional content. Early studies focused on basic eye movement features and manual analysis. For example,  Li et al. conducted one of the first systematic studies using eye-tracking data for depression detection, analyzing fixations, saccades, and pupil size from 34 participants \cite{li2016classification}. Their method achieved 81\% accuracy, but was limited by the small dataset size. Zeng et al. \cite{zeng2018study} expanded this work by studying eye movement patterns during free-viewing tasks with 36 participants, reaching an accuracy of 76.04\%, although still constrained by limited data.
Pan et al. demonstrated that combining reaction time and eye movement data significantly improved results \cite{pan2019depression}. Using a larger dataset of 630 participants, they extracted features based on subjects' responses to positive and negative image attributes. Although the method achieved 72.22\% accuracy, the larger dataset unveiled challenges in maintaining high accuracy across a more diverse population.

Lin et al. \cite{lin2021depression} addressed these limitations by introducing a novel approach combining eye movement data with image semantic understanding. Their study of 181 participants (106 depressed, 75 non-depressed) utilized the Open Affective Standardized Image Set (OASIS) (\cite{kurdi2017introducing}) image library, presenting paired positive and negative emotional images while recording eye movements. By combining visual saliency with semantic image analysis, they achieved an improved accuracy of 90.06\%, demonstrating that understanding where subjects look and what they are looking at provides better insight into depressive states. However, they dealt with only the first three fixation points, which may not always be conclusive, given how the nature of the static images shown to the subjects can vary. Opting for more fixation points by allowing the image to be shown longer could have helped to create better fixation and saliency maps.

Another notable approach by Zhu et al. \cite{zhu2019toward} proposed a Content-Based Ensemble Method (CBEM) that utilized raw eye tracking metrics and EEG data. Their method analyzed 87 direct eye-movement features, including fixations, saccades, pupil size, and dwell time, without focusing on visual attention patterns. Using an ensemble classification approach that divided data based on emotion types, they achieved 82.5\% accuracy with eye tracking data and 92.73\% accuracy with EEG data on a smaller dataset of 36 participants. Although this method did not analyze the semantic content of viewed images, its use of multiple bio-signals and ensemble learning demonstrated the potential of combining different physiological measurements for depression detection. However, their limited dataset size and lack of detailed feature descriptions make it difficult to fully assess the generalizability of the method.

These computer vision-based approaches have revealed consistent patterns: depressed individuals tend to focus more intensely on specific regions (particularly faces in images), whereas healthy individuals exhibit broader scanning patterns that include environmental elements. However, eye movement patterns can be influenced by various factors unrelated to depression such as fatigue, medication, and individual viewing habits. Additionally, the requirement for specialized eye tracking equipment can limit the widespread application of these methods in clinical settings. While our approach doesn't eliminate these fundamental challenges of eye-tracking, our research addresses these limitations by incorporating additional modalities—facial expression and audio analysis—to provide complementary information when eye-tracking data alone might be insufficient or compromised.

\subsection{Depression detection from audio modality}

Acoustic features can provide valuable insights into the clinical identification and diagnosis of mental states as well as depression. Studies have shown that emotional states can profoundly influence the function and structure of the vocal system, as reflected by the rhythm and prosody of the voice \cite{shin2024use}. Early research utilized feature selection and statistical methods to detect depression. Moore et al. \cite{moore2007critical} collected speech data from depressed patients and controls, applied a feature selection strategy related to prosody and vocal tract, and identified glottal features as consistent depression indicators. Later, machine learning enabled the extraction of more complex and dominant speech features. Ma et al. introduced DeepAudioNet to extract features related to depression from vocal channel using Convolutional Neural Network (CNN) and Long ShortTerm Memory (LSTM) \cite{ma2016depaudionet}. Higuchi et al. collected daily phone call data for monitoring mental health changes and depressive states detection from 1,814 adults. The study found that the speech energy levels of depressive participants were significantly lower than those of healthy participants \cite{higuchi2020effectiveness}. Wang et al. utilized voice acoustic features to train an Artificial Neural Network (ANN) for depression score prediction \cite{wang2023fast}. Abbas et al.  studied the effects of antidepressant therapy on 18 MDD subjects using visual and auditory data. Smartphone tasks capture facial expressions, vocal patterns, and head movements, revealing significant changes in speech activity with treatment initiation \cite{abbas2021remote}. The work by Huang et al. \cite{huang2020domain} proposed multiple adaptation strategies that enhance pre-trained models based on a dilated CNN framework, improving depression detection from audio recordings, both free speech and directed speech, in controlled and natural environments. Experiments on two depression datasets demonstrated that adapting CNN feature representations to environmental changes and increasing data quantity and quality during pre-training significantly boost performance. Bilal et al. \cite{bilal2022predicting} collected information by smartphone application and developed a predictive model utilizing voice acoustic features such as pitch, speed, and timing to predict perinatal depression. 

Detecting depression in audio data is a promising but challenging area of research. Symptoms such as monotone speech, slower speech rate, or reduced pitch variability may not always be present or overlap with other conditions (e.g., fatigue and stress). Speech patterns and emotional expressions vary across cultures and languages, making it difficult to generalize models trained on one population to another.

\subsection{Depression detection from visual modality}

Researchers have explored the intricate relationship between emotion and the identification of depression, with significant insights drawn from studies like the one by Joormann and Gotlib \cite{joormann2010emotion}. This study established that depression is associated with deficits in cognitive control, particularly in the domain of emotion regulation, where individuals exhibit pronounced difficulty in inhibiting the processing of negative information. This cognitive bias manifests as a persistent tendency to dwell on negative thoughts and emotions, influencing observable behaviors and expressions. Machine learning capitalizes on this connection by analyzing patterns in diverse data sources, such as voice, text, and physiological signals, to detect subtle indicators of depression. For instance, the negative bias highlighted by Joormann and Gotlib may surface in the use of pessimistic language or altered vocal tones in speech, as well as in physiological responses tied to emotional dysregulation. Machine learning facilitates early detection, personalized treatment, and enhanced diagnostic accuracy using training algorithms to recognize these patterns. Furthermore, extensive research has leveraged vast amounts of user-generated social media data, where linguistic and behavioral cues reflective of cognitive and emotional challenges can be monitored, offering a scalable approach for depression detection and intervention. For example, Reece et al. \cite{reece2017instagram} analyzed Instagram photos, finding that depressed users tend to post images that are bluer, grayer, and darker, achieving an accuracy of 70\% in detecting depression.

Facial emotion recognition (FER) has emerged as a robust modality for detecting depression, leveraging the visible emotional cues tied to cognitive biases. Du et al. \cite{du2019encoding} introduced an attention mechanism to extract discriminating visual information from facial features, enhancing the detection of subtle emotional shifts. Similarly, Hu et al. \cite{hu2023detecting} collected facial expressions from 62 participants while viewing visual stimuli and achieved a reliable classification between healthy and depressed subjects. Islam et al. \cite{Islam2024FacePsy} developed FacePsy, an open-source mobile system that uses a phone’s camera to analyze facial micro-features—such as Action Units (AUs), eye movements, and head gestures—during daily use, reporting an AUROC of 81\% in detecting depressive episodes among 25 participants. Expanding this scope, Yang et al. \cite{yang2019facial} employed a deep learning framework to analyze dynamic facial expressions in video sequences, capturing temporal patterns, such as reduced smile duration and eye contact, which improved detection sensitivity. Zhou et al. \cite{zhou2020depression} proposed a Deep Convolutional Neural Network (DCNN) with a Global Average Pooling layer, focusing on salient facial regions like the eyes and forehead to assess depression severity, achieving high accuracy in controlled settings.

Additional advancements in FER further refine its diagnostic potential. Guo et al. \cite{guo2022multi} introduced a multi-task learning model that simultaneously detects depression and estimates its severity using facial landmarks and expression dynamics, outperforming single-task approaches in real-world datasets. Zhu et al. \cite{zhu2018automated} developed an automated system integrating 3D facial modeling with FER, revealing that subtle depth changes in facial expressions—such as flattened affect—correlate strongly with depressive states. These studies highlight FER’s ability to capture both static and dynamic emotional markers, offering a comprehensive view of depression’s behavioral footprint. He et al. \cite{he2016deep} applied a residual learning framework to enhance FER, demonstrating that deep architectures excel at identifying nuanced emotional patterns linked to mental health conditions, while Jan et al. \cite{jan2018automatic} combined CNNs with temporal feature analysis to track expression changes over time, adding a longitudinal dimension to depression detection.

FER’s integration into a multimodal approach significantly enhances depression detection by combining independent and overlapping features across data types. Standalone, FER provides unique insights through micro-expressions and gaze patterns that reflect emotional suppression or negativity specific to depression. When paired with modalities such as speech or text, it overlaps with features such as slow speech tempo or negative word choice, reinforcing diagnostic confidence. This synergy leverages FER’s strengths—its non-invasive, real-time applicability—while compensating for limitations in other modalities, such as context dependency in text analysis. By synthesizing these diverse signals, a multimodal framework achieves greater robustness and reliability, paving the way for scalable, personalized mental health solutions.

\subsection{Multi-modal approach}

\

 Applying multiple data modalities and their fusion in machine learning and deep learning applications has yielded more optimal outcomes across many fields. These systems can capture a broader range of behavioral cues and improve the robustness of depression detection. Work by Stasak et al. \cite{stasak2019investigation} improved the classification between mild and major depression by integrating emotional and acoustic features and discovered that combining emotional scores extracted from linguistic stress and vowel articulation, obtained from speech, can enhance the accuracy of depression recognition using voice analysis alone. Yin et al. \cite{yin2019multi} proposed a multi-modal approach using a layered recurring neural network to combine visual, auditory, and textual data for detecting depression symptoms. Further efforts by Niu et al. \cite{niu2023multimodal} improved this approach by implementing a novel spatiotemporal attention network on audio and visual data for better performance. 



Recent research has extensively explored multi-modal approaches for depression detection, leveraging diverse physiological data sources such as heart rate variability, electroencephalography (EEG), and cortisol levels to improve diagnostic accuracy and provide more comprehensive insights into mental health conditions. Katyal et al. \cite{katyal2014eeg} combined EEG signals with facial video features for improved performance.

Zhou et al. \cite{zhou2024multi} proposed a Multi Fine-Grained Fusion Network (MFFNet) that effectively captures behavioral changes at different scales by fusing speech and text features from interviews. Their model employs a Multi-Scale Fastformer to capture correlations between temporal features, and uses a Gated Fusion Module to dynamically weight and combine features from different modalities. Their experiments on two depression interview datasets showed superior performance compared to unimodal approaches. \cite{cai2020feature} explored the feature-level fusion of EEG data under different audio stimuli (neutral, negative, and positive), demonstrating that multimodal EEG feature fusion could achieve higher classification accuracy (86.98\%) compared to single modality approaches. They used linear and nonlinear EEG signal features and employed genetic algorithms for feature selection.

In their works, Alghowinem et al. \cite{alghowinem2018multimodal} developed a multimodal approach combining paralinguistic features, head pose, and eye-gaze behaviors. Their fusion analysis showed that feature-level fusion performed best with up to 88\% accuracy, demonstrating these behavioral indicators' complementary nature for depression detection. In another work, Chen et al. \cite{chen2023ms2} proposed MS2-GNN, a novel graph neural network framework that explores modal-shared and modal-specific characteristics among various psychophysiological modalities while investigating potential relationships between subjects. Their approach handles heterogeneity and homogeneity among modalities while maintaining inter-class separability and intra-class compactness. 

Zhou et al. \cite{zhang2024multimodal} introduced AVTF-TBN, an innovative framework combining audio, video, and text features through a three-branch network architecture. Their approach utilizes specialized branches for each modality and implements a multimodal fusion module based on attention mechanisms, achieving superior performance in real-world scenarios. Tao et al. \cite{tao2023depmstat} developed DepMSTAT, which integrates spatio-temporal attention mechanisms with multimodal fusion strategies. Their framework effectively captures the temporal dynamics of facial expressions and speech patterns while maintaining the semantic consistency of textual information through a sophisticated fusion network.

\subsection{Available Datasets}

Effective depression recognition requires sufficient data to train discriminative models, but data collection is challenging due to the sensitive nature of depression. Various research groups have created their own databases to study depression assessment tools. Here we introduce the commonly used datasets on depression detection. The Distress Analysis Interview Corpus/Wizard-of-Oz (DAIC-WOZ) dataset \cite{gratch2014DAIC} contains voice and text samples from interviews with 189 healthy and control subjects, along with their PHQ-8 \cite{kroenke2009phq8} depression questionnaires to support the diagnosis of psychological disorders such as post-traumatic stress disorder, depression, and anxiety. This dataset is widely used in research for text-based, voice-based, and multimodal depression detection studies. 

The E-DAIC (Extended Distress Analysis Interview Corpus) is an extension of the original DAIC-WOZ dataset that includes audio, video, and textual data collected from clinical interviews conducted with participants, some of whom have been diagnosed with depression or other mental health conditions. This dataset comprises more than 73 hours of interview data from 275 subjects. 

The AVEC2013 dataset is a collection from the audiovisual depressive corpus. It contains 340 video recordings of 292 individuals interacting with a computer system. The participants in this study ranged from 18 to 63 years old, with an average age of 31.5 years \cite{valstar2013avec}. 

The AVEC2014 dataset is an updated version of  AVEC2013 containing more data samples \cite{valstar2014avec}.While the AVEC2013 focuses primarily on emotion recognition, the AVEC2014 includes audio, video, and textual transcriptions of interviews, along with depression severity labels based on PHQ-8.

The Multi-modal Open Dataset for Mental-disorder Analysis (MODMA) includes EEG and speech data from clinically depressed patients and matched controls, annotated by professional psychiatrists \cite{cai2022_MODMA_dataset}. It features 128-electrode EEG signals from 53 participants, 3-electrode EEG signals from 55 participants, and audio recordings from 52 participants captured during interviews, reading tasks, and picture descriptions.

These datasets have facilitated the ongoing research to develop automated tools for depression detection. Unfortunately, they lack diversity in age, culture, ethnicity, and accent, which are critical in modeling real-world emotional or depressive expressions. Participants are mostly young adults and native English speakers, limiting demographic diversity (e.g., age, ethnicity, cultural background). Additionally, none of these datasets contains eye-tracking data, whose usefulness has been proven in diagnosing cognitive disorders. So we have created our own dataset by collecting audio, video, and eye tracking modalities from native citizens. The detailed data collection setup and inclusion criteria have been explained in the \textbf{Study Design} section.   

\section{Tripartite Data Approach}

\begin{table*}[!ht]
\centering
\begin{threeparttable}
\caption{Descriptive Statistics for Emotions in Depression}
\small
\begin{tabular}{|l|cc|cc|cc|}
\hline
\rowcolor[rgb]{0.6,0.682,1} & \multicolumn{2}{c|}{No Depression} & \multicolumn{2}{c|}{Mild to Moderate Depression} & \multicolumn{2}{c|}{Severe Depression} \\
\hline
\rowcolor[rgb]{0.6,0.682,1} Emotion & Mean & Std. & Mean & Std. & Mean & Std. \\
\hline
angry*** & 0.031 & 0.032 & 0.055 & 0.031 & 0.089 & 0.057 \\
disgust & 0.000 & 0.001 & 0.001 & 0.001 & 0.003 & 0.009 \\
fear** & 0.039 & 0.028 & 0.060 & 0.034 & 0.095 & 0.070 \\
happy & 0.080 & 0.078 & 0.117 & 0.113 & 0.065 & 0.108 \\
sad*** & 0.055 & 0.055 & 0.110 & 0.054 & 0.157 & 0.077 \\
surprise*** & 0.018 & 0.018 & 0.008 & 0.005 & 0.040 & 0.031 \\
neutral*** & 0.777 & 0.135 & 0.651 & 0.118 & 0.551 & 0.178 \\
\hline
\end{tabular}
\begin{tablenotes}
\small
\item * p-value $<$ 0.05; ** p-value $<$ 0.01; *** p-value $<$ 0.001
\end{tablenotes}
\label{tab:emotion_features}
\end{threeparttable}
\end{table*}

\subsection{Eye Tracking}

Eye tracking captures gaze patterns by measuring the precise locations where individuals fixate their visual attention. Using infrared cameras, eye trackers detect and record pupil position, enabling analysis of fixation patterns, saccadic movements, and pupil dilation. These eye movement metrics provide objective measures of attentional processes that are largely involuntary and resistant to conscious manipulation \cite{rahman2021classifying}.

For depression detection, eye tracking is particularly valuable because individuals with MDD exhibit distinct gaze patterns that reflect underlying attentional biases and emotional dysregulation. As established in our literature review, depressed individuals demonstrate prolonged fixation on negative stimuli and reduced attention to positive cues, making eye tracking a clinically relevant modality for MDD assessment. The technology's non-invasive nature, relatively low cost, and ability to capture subconscious attentional processes without requiring verbal responses make it well-suited for objective depression screening \cite{ahonniska2018assessing}.

\subsection{Facial Expression from Video Data}


Facial expressions provide observable manifestations of emotional states through coordinated movements of facial muscles. These nonverbal signals convey affective information that can be quantified and analyzed to assess emotional functioning. In the context of depression, facial expression analysis is particularly relevant because MDD is characterized by distinct patterns of reduced expressivity, including decreased smiling, heightened displays of sadness, and affective flattening—observable markers that reflect the emotional withdrawal and blunted affect central to the disorder \cite{girard2014nonverbal}.

Modern computer vision approaches, particularly deep learning-based Facial Emotion Recognition (FER) systems, enable automated extraction and quantification of emotional expressions from video data. These systems employ convolutional neural networks trained on large datasets of facial images to detect and classify emotions across multiple categories (e.g., happiness, sadness, anger, fear, disgust, surprise, and neutral). By analyzing frame-by-frame facial configurations and their temporal dynamics, FER systems provide objective measures of emotional expression patterns that complement self-report assessments. Unlike subjective clinical observations, automated facial analysis offers consistent, quantifiable metrics of affective display, making it a valuable tool for depression assessment \cite{pampouchidou_2017_depression_visual_cues}.

For the extraction of emotional information from the subject`s facial expressions, we employ FER \cite{fer}, an open-source Python library that utilizes a pre-trained face expression recognition model. The FER model processes individual video frames, predicting the emotional state of detected faces across seven categories: anger, disgust, fear, happiness, neutrality, sadness, and surprise. Our approach mirrors the audio feature extraction process; we first extract emotion scores for each frame within a subject's interview with the PHQ-9 administrators, then aggregate these frame-level features to the video level by computing the mean score for each emotion category \cite{koldyk2018detecting}. This process yields seven visual features to complement the eight audio features previously extracted. It's important to note that our analysis is confined to frames where a single face is detected, ensuring the clarity and reliability of our emotional assessments \cite{li2020deep}.

From the video recordings of the subjects' interactions, the features extracted are summarized in Table \ref{tab:emotion_features}.

\textbf{Anger:} Our results indicate a significant increase in anger expression as depression severity increases. This aligns with findings by \cite{carvalho2013anger}, who noted that individuals with depression often exhibit more anger, potentially due to feelings of frustration and helplessness associated with their condition.

\textbf{Disgust:} While our data shows a slight increase in disgust expression with depression severity, the difference is not statistically significant. This is consistent with \cite{zwick2017differences}, who found that disgust recognition, rather than expression, was more notably affected in depressed individuals.

\textbf{Fear:} We observed a significant increase in fear expression correlated with depression severity. This supports the findings from \cite{dalili2015meta}, that depressed individuals tend to show heightened sensitivity to negative emotions, including fear.

\textbf{Happiness:} Interestingly, our data shows a non-linear relationship between happiness expression and depression severity, with the mild to moderate group showing the highest mean. This complex relationship is echoed in \cite{girard2014nonverbal}, suggesting that some depressed individuals may engage in ``smiling depression" as a coping mechanism.

\textbf{Sadness:} As expected, sadness expression significantly increases with depression severity. This is widely supported in literature, including \cite{rottenberg2005mood}, who noted persistent sad mood as a hallmark of depression.

\textbf{Surprise:} Our data shows a significant non-linear relationship between surprise expression and depression severity. This complex pattern might be related to findings by \cite{leppanen2004emotional}, who noted altered emotional processing in depression, affecting recognition and expression of emotions like surprise.

\textbf{Neutral:} We observed a significant decrease in neutral expression as depression severity increases. This aligns with \cite{reed2007inhibition}, who found that depressed individuals often have difficulties in maintaining neutral expressions, tending towards negative emotional displays.

These findings collectively support the notion that facial expressions can serve as valuable indicators of depression severity, aligning with broader research on emotional expression in mental health contexts \cite{cohn2009detecting}.

\subsection{Acoustic Features from Audio Data}

\begin{table}[htbp]
\centering
\caption{Statistical Analysis of Top Five Features from PCA component 1}
\label{tab:mel_features_stats}
\begin{tabular}{lcccccc}
\toprule
Feature & \multicolumn{2}{c}{No Dep.} & \multicolumn{2}{c}{Mild-Mod.} & \multicolumn{2}{c}{Severe} \\
\cmidrule(lr){2-3} \cmidrule(lr){4-5} \cmidrule(lr){6-7}
& Mean & Std & Mean & Std & Mean & Std \\
\midrule
mel 38$^{*}$ & 0.15 & 0.14 & 0.36 & 0.25 & 0.36 & 0.23 \\
mel 39$^{**}$ & 0.12 & 0.11 & 0.34 & 0.23 & 0.33 & 0.20 \\
mel 42$^{***}$ & 0.08 & 0.10 & 0.32 & 0.25 & 0.32 & 0.22 \\
mel 43$^{**}$ & 0.09 & 0.10 & 0.31 & 0.24 & 0.30 & 0.21 \\
mel 52$^{**}$ & 0.04 & 0.04 & 0.14 & 0.13 & 0.13 & 0.09 \\
\bottomrule
\end{tabular}
\begin{tablenotes}
\small
\item Significance: $^{*}p < 0.05$, $^{**}p < 0.01$, $^{***}p < 0.001$.
\end{tablenotes}
\end{table}

Although facial expressions are often considered a primary conduit to understanding human emotions, vocal expressions also serve as a significant indicator of emotional states. Beyond the denotative content of spoken language, paralinguistic features, such as intensity (loudness), pitch, and speech irregularities, including vocal tremors, convey a rich tapestry of emotional nuances. Additionally, the temporal aspects of speech, including sentence length and rhythm, juxtaposed with non-lexical vocalizations like sighs, provide deeper insight into the speaker's emotional psyche \cite{almaghrabi_2023_depression_acoustic_audio}. 

Despite the richness of this data, the ability of individuals to consciously discern and interpret these subtle acoustic markers is typically suboptimal. The expertise required for accurate emotional assessment from speech is often reserved for trained audio analysts. Within this context, advancements in artificial intelligence, specifically deep learning algorithms, have shown promise in bridging this gap. These algorithms excel at extracting complex patterns from extensive datasets—in this instance, audio samples correlated with emotional states, enabling the prediction of emotions with considerable accuracy \cite{xia2024_multimodal_depression}. These sophisticated models can identify distinctive patterns corresponding to various emotional expressions by meticulously examining large-scale acoustic data. The implications of this technological feat are vast, offering transformative potential across disciplines, from enhancing human-computer interaction to supporting mental health professionals in diagnostic procedures. The progression in deep learning applied to emotional recognition is discussed in detail in \cite{el2011survey}, whose study provides evidence for the significant predictive power of these algorithms.


The audio recordings from the patients, which were converted into the .wav format, were cleaned of any background noise using Ultimate Voice Remover (UVR) 5.6.0 \cite{uvr2023}, an advanced open-source tool designed for vocal isolation and removal from audio tracks. Afterward, each of the audio files was subjected to speech diarization. This utilized the PyAnnote library \cite{bredin2020pyannote} for speaker diarization, employing the pre-trained ``pyannote/speaker-diarization-3.1" model.

For the extraction of acoustic features from the audio recordings, we utilize OpenSmile, an open-source audio processing toolkit, in conjunction with the extended Geneva Minimalistic Acoustic Parameter Set (eGeMAPS) \cite{eyben2015geneva}. Our selection of acoustic characteristics is informed by previous studies on the vocal patterns of individuals with depression \cite{al2018detecting,ellgring1996vocal,fuller1992vocal,jia2019detecting,solomon2015vocal,tamarit2008study,vicsi2012depression,wang2019detecting}. We extract eight low-level acoustic descriptors (LLDs), including Loudness, Fundamental Frequency (F0), and Spectral Flux, among others.

Table \ref{tab:acoustic_features} gives a summary of the acoustic features extracted from the audio recordings of the interviews with the subjects. We further explain the features and the trends noticed in these features:

\begin{table*}[!ht]
\centering
\begin{threeparttable}
\caption{Descriptive Statistics for Audio Features in Depression}
\label{tab:acoustic_features}
\small
\begin{tabular}{|l|cc|cc|cc|}
\hline
\rowcolor[rgb]{0.6,0.682,1} & \multicolumn{2}{c|}{No Depression} & \multicolumn{2}{c|}{Mild to Moderate Depression} & \multicolumn{2}{c|}{Severe Depression} \\
\hline
\rowcolor[rgb]{0.6,0.682,1} Audio Feature & Mean & Std. & Mean & Std. & Mean & Std. \\
\hline
Loudness* & 0.981 & 0.207 & 1.205 & 0.364 & 1.192 & 0.341 \\
F0*** & 28.480 & 3.070 & 31.593 & 4.060 & 34.571 & 3.806 \\
HNR*** & 3.630 & 1.243 & 4.300 & 1.904 & 5.986 & 1.815 \\
Jitter** & 0.045 & 0.013 & 0.0466 & 0.012 & 0.039 & 0.010 \\
Shimmer*** & 1.388 & 0.100 & 1.353 & 0.110 & 1.263 & 0.111 \\
F2** & 1519.730 & 81.623 & 1548.213 & 85.155 & 1590.002 & 89.143 \\
Hammarberg Index*** & 0.077 & 0.021 & 0.094 & 0.018 & 0.100 & 0.014 \\
Spectral Flux* & 0.728 & 0.231 & 0.925 & 0.329 & 0.877 & 0.282 \\
\hline
\end{tabular}
\begin{tablenotes}
\small
\item * p-value $<$ 0.05; ** p-value $<$ 0.01; *** p-value $<$ 0.001
\end{tablenotes}
\end{threeparttable}
\end{table*}

\textbf{Loudness and Fundamental Frequency (F0):} Our analysis reveals that both loudness and fundamental frequency (F0) are significantly higher in the depression group compared to the non-depression group ($p < 0.05$). This finding aligns with previous research by \cite{mundt2012voice}, who observed changes in speech amplitude and pitch associated with depression severity. The increase in loudness may seem counterintuitive, as depression is often associated with reduced energy. However, \cite{ciftci2018analysis} suggests that this could be related to the effort required to maintain normal speech patterns under depressive conditions. The elevation in F0 might reflect increased tension in the vocal folds, a physiological response potentially linked to the stress and anxiety often tag along with depression \cite{cannizzaro2004voice}.

\textbf{Harmonics-to-Noise Ratio (HNR):} Our results indicate a significant increase in HNR with depression severity ($p < 0.001$). This finding is particularly interesting as it contrasts with some previous studies. For instance, \cite{alghowinem2013comparative} reported a decrease in HNR among depressed individuals. Our observation of increased HNR might suggest a compensatory mechanism in depressed speech, where individuals unconsciously strive for clearer articulation to counteract perceived communicative deficits. This discrepancy highlights the complex nature of vocal changes in depression and warrants further investigation into potential subtype-specific effects or methodological differences across studies.

\textbf{Jitter and Shimmer:} Contrary to some previous findings, our data shows a significant decrease in both jitter and shimmer with increasing depression severity ($p < 0.01$ for both). This stands in contrast to studies like \cite{cannizzaro2004voice} and \cite{ozkanca2019evaluation}, which reported increases in these perturbation measures. Jitter, which quantifies cycle-to-cycle variations in fundamental frequency, and shimmer, which measures amplitude variations, are typically associated with perceived voice roughness. Our unexpected finding of decreased jitter and shimmer might indicate a form of vocal control or tension in depressed speech, possibly reflecting the flattened affect often observed in clinical depression. These results emphasize the need for nuanced interpretation of acoustic features in the context of depression.

\textbf{Second Formant (F2):} Our analysis reveals a significant increase in F2 frequency with depression severity ($p < 0.01$). This aligns with findings from \cite{scherer2015self}, who noted alterations in formant frequencies in depressed speech. F2 is particularly associated with the front-back position of the tongue and the degree of lip rounding. An increase in F2 could suggest a tendency towards more frontal articulation in depressed speech, possibly reflecting changes in muscle tension or reduced articulatory effort. This shift in F2 may contribute to the perception of ``flatter" or less animated speech often described in depressed individuals.

\textbf{Hammarberg Index:} We observe a significant increase in the Hammarberg Index with depression severity ($p < 0.001$). This index, which measures the difference in energy between lower (0-2 kHz) and higher (2-5 kHz) frequency bands, has been less commonly reported in depression studies. However, our findings are consistent with the work of \cite{vicsi2012depression}, who utilized this measure in their analysis of depressed speech. The increase in Hammarberg Index suggests a shift in spectral energy distribution, potentially reflecting changes in vocal tract tension or alterations in phonation patterns associated with depressive states. This could contribute to the perceived ``strained" or ``tense" quality sometimes noted in the speech of depressed individuals.

\textbf{Spectral Flux:} Our results show a significant increase in Spectral Flux with depression severity ($p < 0.05$). This measure, which quantifies the frame-to-frame spectral change in the speech signal, aligns with findings reported in the comprehensive review by \cite{cummins2015review}. The increase in Spectral Flux might indicate greater instability or variability in the spectral characteristics of depressed speech. This could be related to subtle irregularities in vocal fold vibration or changes in articulatory precision, both of which might be influenced by the psychomotor changes associated with depression. The elevated Spectral Flux might contribute to the perception of depressed speech as less ``smooth" or more ``effortful" compared to non-depressed speech.

\subsection{Deep Learning-Based Feature Representations}

\ref{tab:acoustic_features}, we extract complementary spectral representations that enable our deep learning model to automatically learn hierarchical patterns related to depression. While the hand-crafted features (F0, jitter, shimmer, HNR, etc.) provide interpretable, clinically meaningful measures with established statistical significance, modern neural architectures benefit from accessing richer, less-compressed representations of the audio signal \cite{purwins2019deep, trigeorgis2016adieu}. Rather than discarding the insights from traditional features, we leverage spectral representations that implicitly encode similar acoustic properties while preserving additional information that may be relevant for depression detection \cite{ma2016depaudionet,zhao2019speech}. Specifically, we compute three complementary feature sets: chroma, Mel spectrograms, and MFCCs, with each capturing different aspects of the speech signal that relate to the acoustic parameters discussed above.

\textbf{Chroma feature} represents the distribution of spectral energy across the 12 pitch classes (C, C\#, D, ..., B), effectively combining the harmonic content and providing octave-invariant information that relates strongly to pitch (F0) and indirectly to vocal tract characteristics (F2) and loudness fluctuations. A chromagram $C(t, k)$ at time $t$ and pitch class $k$ is formally computed by summing the Short-Time Fourier Transform (STFT) magnitude $|X(t, f)|$ over all frequencies $f$ mapped to chroma bin $k$:
\[
C(t, k) = \sum_{f \in \mathcal{F}_k} |X(t, f)|
\]
where $\mathcal{F}_k$ denotes frequencies assigned to chroma bin $k$ \cite{muller2005chroma,mueller2011chord}. By capturing the pitch class distribution, chroma features provide a complementary view of the fundamental frequency variations (F0) observed in Table \ref{tab:acoustic_features}, while maintaining robustness to octave shifts that may occur due to speaker variability.

\textbf{Mel spectrograms} project the power spectrum onto the Mel scale, approximating human auditory sensitivity. Formally, the Mel spectrogram $M(t, m)$ at time $t$ and Mel band $m$ is:
\[
M(t, m) = \sum_{f \in \text{Mel}(m)} |X(t, f)|^2
\]
This captures spectral envelope information strongly related to loudness (by total energy), spectral flux (temporal change of $M(t, :)$), and can reflect shimmer and HNR through local variance and periodicity in the bands \cite{shao2018cleanmel,ravin2023genai}. The time-frequency resolution of Mel spectrograms preserves the temporal dynamics of the acoustic features analyzed earlier, enabling the neural network to learn patterns in loudness fluctuations, spectral changes, and harmonic structures that correlate with depression severity.

\textbf{MFCCs} (Mel Frequency Cepstral Coefficients) compact the Mel spectrogram by taking the discrete cosine transform (DCT) of the log Mel power. The $n$-th MFCC at time $t$ is:
\[
\text{MFCC}_n(t) = \sum_{m} \log M(t, m) \cos\left[\frac{\pi n}{M}(m + 0.5)\right]
\]
MFCCs therefore encode spectral envelope and are correlated with timbral aspects (shaped by vocal tract resonances like F2), as well as summarising periodic changes relevant to F0 and HNR \cite{telmem2025impact,srivastava2013speech,logan2000mel}. The lower-order MFCCs capture the broad spectral shape influenced by formant frequencies (including F2), while higher-order coefficients encode finer spectral details related to voice quality measures such as jitter and shimmer.

By merging pitch class (chroma), perceptual spectrum (Mel), and envelope/timbre (MFCCs), these features allow a deep learning model to internalize traditional acoustic parameters, including F0, loudness, jitter, shimmer, and spectral flux, holistically and efficiently. Rather than relying on single-valued summaries, these representations preserve the temporal evolution and spectral richness of the speech signal, blending rich acoustic information into three learnable domains well-suited for end-to-end modeling \cite{korzeniowski2016deep,mueller2018stft}. This approach enables the network to automatically discover complex relationships between acoustic patterns and depression severity that may not be captured by traditional hand-crafted features alone, while still maintaining implicit connections to the clinically interpretable measures presented in Table \ref{tab:acoustic_features}.
\subsubsection{Choosing Deep Learning Features Over Hand-Crafted Acoustic Features}

While traditional acoustic features such as Loudness, F0, HNR, jitter, shimmer, F2, Hammarberg Index and Spectral Flux have demonstrated statistical significance in distinguishing depression severity (as shown in Table \ref{tab:acoustic_features}), modern deep learning approaches benefit substantially from utilizing low-level spectral representations rather than these hand-crafted features. We discuss out rationale for employing Mel spectrograms and related representations in our neural network architecture down below:

\textbf{Representation Learning and Feature Hierarchies:} Deep neural networks possess the inherent capability to automatically learn hierarchical feature representations directly from raw or minimally processed audio data \cite{purwins2019deep}. Hand-crafted features like jitter and shimmer were originally designed for classical machine learning approaches (e.g., Support Vector Machines, Random Forests) that lack this automatic feature learning capability \cite{schuller2013computational}. When we provide pre-computed acoustic features to a neural network, we are essentially performing feature engineering twice: once through manual acoustic analysis, and again through the network's learned representations. This redundancy can limit the model's ability to discover novel patterns in the data \cite{rajpurkar2017cardiologist}.

\textbf{Information Preservation:} Hand-crafted acoustic features represent highly compressed summaries of the underlying audio signal. For instance, the Hammarberg Index reduces the entire spectral distribution to a single scalar value representing energy differences between frequency bands. While this compression provides interpretability, it inherently discards potentially relevant information that a neural network might leverage for classification \cite{trigeorgis2016adieu}. In contrast, Mel spectrograms preserve the temporal-spectral structure of speech, retaining rich information about formant trajectories, prosodic patterns, and fine-grained spectral variations that may correlate with depression \cite{zhao2019speech}.

\textbf{Mel-Frequency Spectrograms:} Mel spectrograms have emerged as the de facto standard for speech-based deep learning applications, including emotion recognition and clinical speech analysis \cite{gideon2017progressive, ma2016depaudionet}. These representations apply a perceptually-motivated frequency scale that mimics human auditory perception while maintaining sufficient temporal and spectral resolution for neural networks to extract discriminative patterns \cite{hershey2017cnn}. Studies specifically focused on depression detection have demonstrated that convolutional neural networks (CNNs) operating on Mel spectrograms achieve superior performance compared to models trained on traditional acoustic features \cite{ma2016depaudionet, He_2018}. Mel-Frequency Cepstral Coefficients (MFCCs), derived from Mel spectrograms through discrete cosine transformation, provide a compact representation that captures the spectral envelope while reducing dimensionality \cite{ganchev2005comparative}. While MFCCs have been widely employed in speech recognition systems \cite{rabiner1993fundamentals}, recent depression detection studies suggest that full Mel spectrograms often outperform MFCCs when sufficient training data is available, as the additional information retention justifies the increased dimensionality \cite{rejaibi2022mfcc}.

\textbf{Raw Waveform Processing:} End-to-end deep learning models that operate directly on raw audio waveforms represent the most extreme form of automatic feature learning \cite{oord2016wavenet, dai2017very}. These approaches eliminate all manual feature engineering, allowing the network complete freedom to learn task-specific representations from scratch. While computationally intensive, raw waveform models have shown promise in various audio classification tasks \cite{tokozume2017learning}, though they typically require substantially larger datasets than spectrogram-based approaches.

\subsection{Ground Truth Labeling}
The process of ground truth labeling serves as a crucial step in developing an AI tool that helps to provide a reliable methodology. This section delves into the methodologies employed in our Ground Truth Labelling process, emphasizing the screening and diagnosis phases, the tools used, and the significance of an accurate assessment. \textit{Screening Phase:} Our ground truth labelling process begins with a meticulous screening phase conducted by a seasoned clinical physiatrist. They verbally interact with the subjects, at length, to determine the presence of possible mental disorders, with a particular focus on MDD. Participants who are suspected of having even the faintest possibility of depressive disorders is passed onto the next phase. \textit{Diagnosis Phase:} Following the screening phase, the selected individuals proceed to the Diagnosis phase to further refine the labelling process. Here, two professional psychiatrists undertake a detailed examination, employing the commonly used PHQ-9 \cite{martin2006validity} as the primary tool to provide the final ground truth score.


\section{Study Design}


The study design is a case-control study that was conducted between December 2023 and May 2024. The study was approved by the Bangladesh Medical Research Council. Written informed consent was obtained from all subjects.


\subsection{Study Population}

The MDD subjects were recruited from the National Institute of Mental Health and Hospital (NIMH) in Dhaka, Bangladesh. Licensed psychiatrists of the institutions clinically diagnosed the subjects.  According to the PHQ-9 assessment, the control group showed no sign of a depressive disorder. 
The exclusion criteria for all study participants were as follows: Individuals unable to remain seated for a prolonged verbal interaction were considered unfit for our data collection method, which required them to stay seated in a chair for approximately 5-10 minutes during the interaction.
In total, we recruited 103 participants for analysis. The cohort's median (range) age is 25 (17-52) years, with a standard deviation of 8.9 years. The majority of the cohort is males (54.3\%, N=105). The primary inclusion criteria are based on the PHQ-9 scale. The interpretation of the PHQ-9 score is as follows: 1-4: minimal depression, 5-9 mild depression, 10-14: moderate depression, 15-19 moderately severe depression, 20-27 severe depression. For simplicity, we rescaled the PHQ-9 scale to 3 classes, where a score less than 5: mild depression, 5-15: moderate depression, and greater than 15: severe depression. Participants are classified into three groups according to the severity of their depression: mild (PHQ-9 score $\leq$4), moderate (5$\geq$ PHQ-9 score $\leq$ 15) and severe (PHQ-9 score $\geq$15). The 103 participants can be divided into three age groups.

\begin{figure}[!t]
    \centering
    \begin{minipage}[t]{0.48\textwidth}
        \centering
        \resizebox{\linewidth}{!}{%
        \begin{tabular}{|lll|} 
        \hline
        \multicolumn{3}{|l|}{\textbf{Participant Demographic Summary}} \\ 
        \hline
        \textbf{Gender} & \textbf{Number} & \textbf{Percent (\%)} \\
        Male & 57 & 54.29 \\
        Female & 48 & 45.71 \\
        \textbf{PHQ-9 scale Diagnosis} &  &  \\
        Mild  & 28 & 27.19 \\
        Moderate  & 36 & 34.95 \\
        Severe  & 39 & 37.86 \\
        \textbf{Subject Age Group} &  &  \\
        $\leq$ 20 & 20 & 19.04 \\
        21-39 & 71 & 69.52 \\
        $\geq$ 40 & 12 & 11.43 \\
        Mean Age (years) & 28.16 &  \\
        Median Age (years)  & 25 &  \\
        Age Standard Deviation  & 8.92 &  \\
        \hline
        \end{tabular}
        }
        \captionof{table}{Demographic characteristics and PHQ-9 score distribution (N=105)}
        \label{tab:participant_demographic_data}
    \end{minipage}%
    \hfill
    \begin{minipage}[t]{0.48\textwidth}
        \centering
        \includegraphics[width=\textwidth]{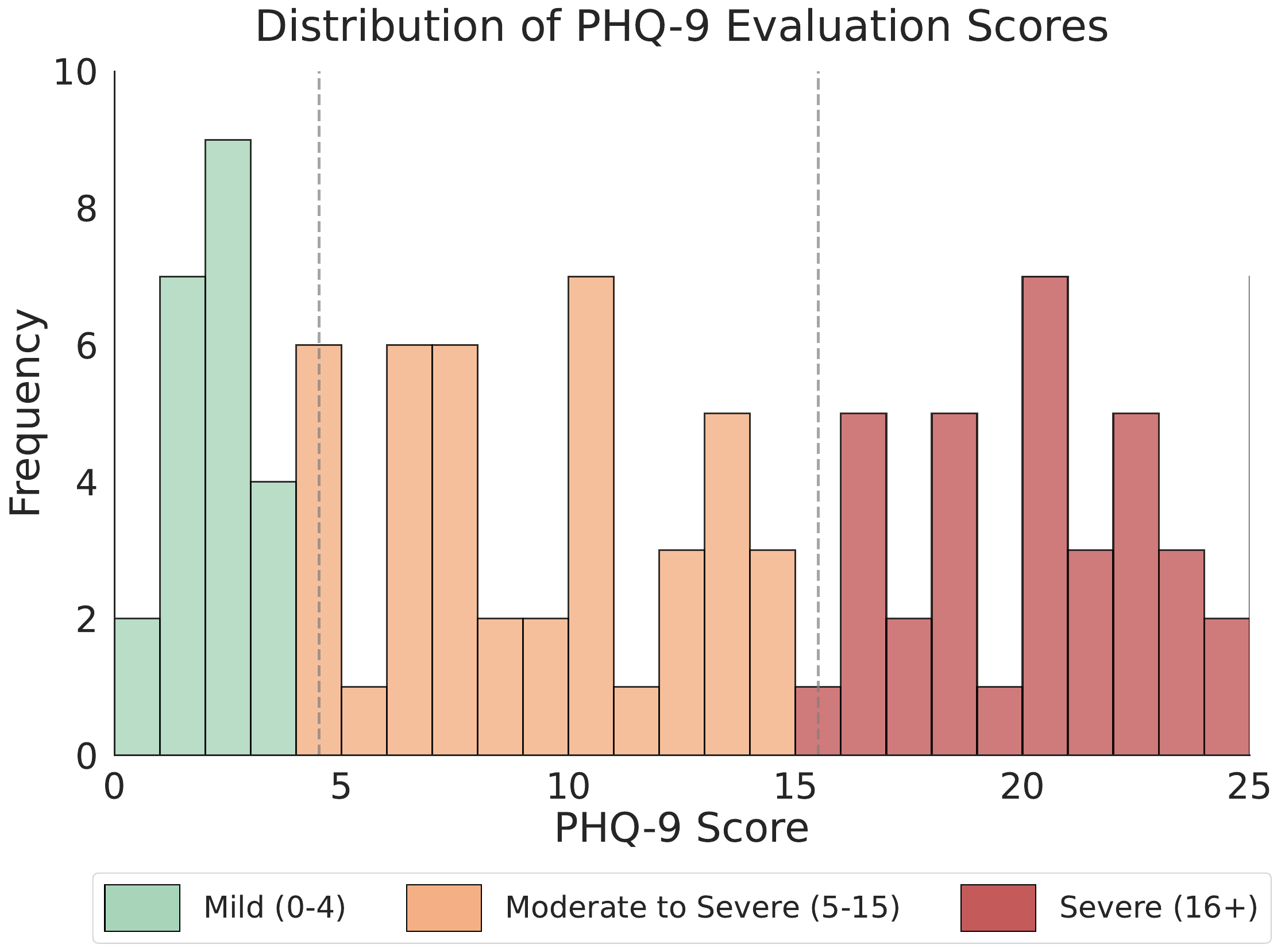}
        \caption{Histogram illustrating the distribution of Patient Health Questionnaire-9 (PHQ-9) scores among subjects}
        \label{fig:score_distribution_histogram}
    \end{minipage}
\end{figure}

The violin plot in Figure  \ref{fig:gender_vs_score_violinplot}  illustrates the distribution of PHQ-9 scores for male and female subjects, highlighting significant differences in depressive symptom prevalence and severity between genders. Males exhibit a distribution skewed toward lower scores, with a median of 6.0 
placing most in the ``mild depression" range or the lower end of "moderate to severe depression." The violin shape for males shows a wide base at lower scores, indicating a concentration of mild cases and fewer individuals with higher scores. In contrast, females have a median score of 17.0 
with most falling into the ``severe depression" category. Their violin shape is more symmetrical and fuller in the upper ranges, suggesting a more even distribution of moderate to severe symptoms. This indicates a higher prevalence and severity of depression among females in the study. Additionally, the consistent width of the female violin across score ranges points to greater variability in symptom severity, further emphasizing the disparity between genders. Overall, the plot reveals that females experience higher and more varied levels of depressive symptoms compared to males.

\subsection{Data Collection Setup}

\begin{figure}[!tbp]
    \centering   \includegraphics[width=\linewidth]{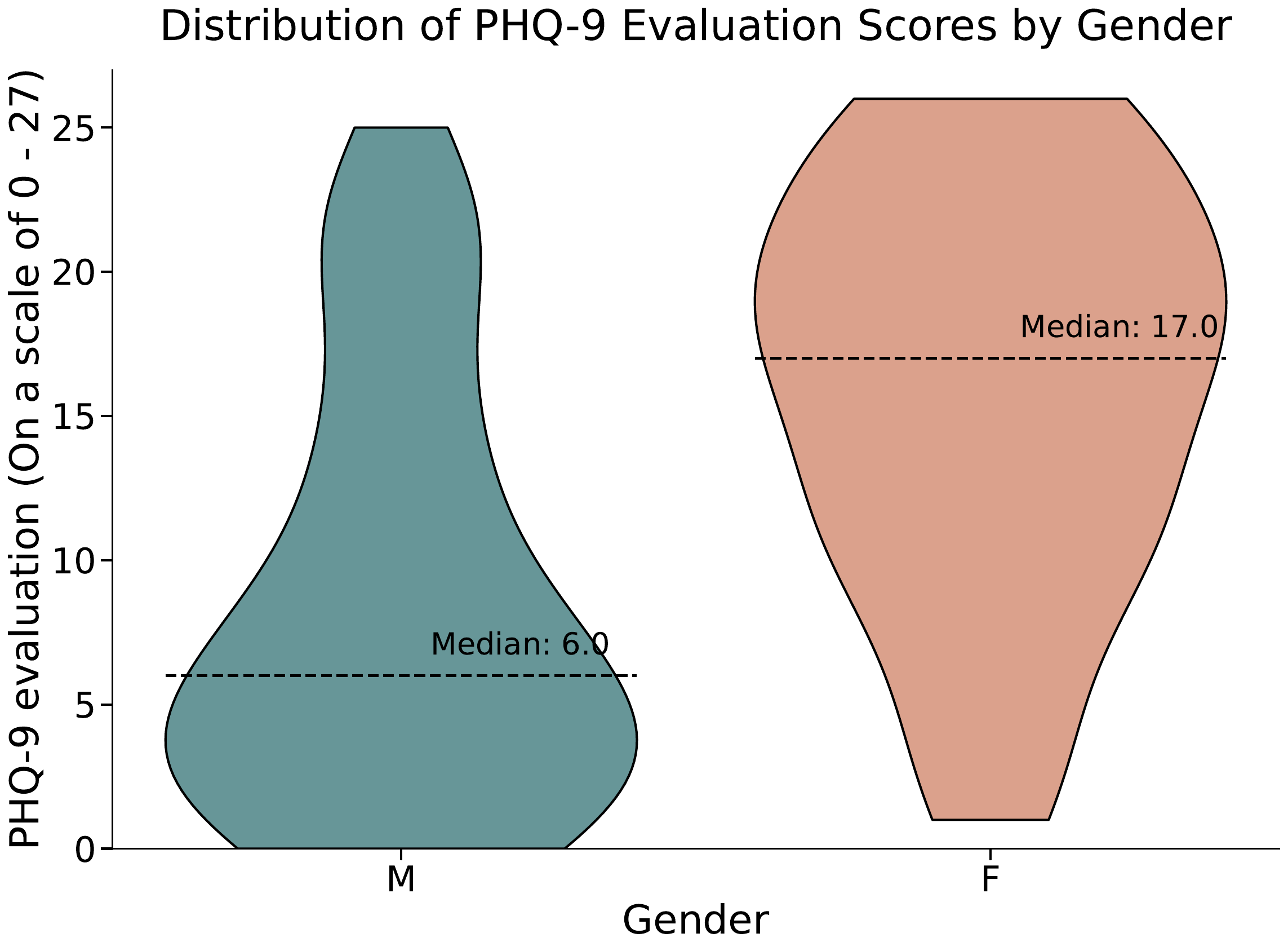}
    \caption{Violinplot illustrating the distribution of Patient Health Questionnaire-9 (PHQ-9) scores between male and female subjects.}
\label{fig:gender_vs_score_violinplot}
\end{figure}

The tripartite nature of the data requires that the three modalities be collected on-site within a reasonably short span of time, to ensure data alignment and quality \cite{Wei2020CMHAD}. Figure \ref{fig:Tripartite_Data_Collection_Illustration} illustrates a high-level workflow of the data collection from both our collection sites: The National Institute of Mental Health (NIMH) and the University of Dhaka (DU).

\begin{figure}[!h]
    \centering
  \includegraphics[width=1.03\linewidth]{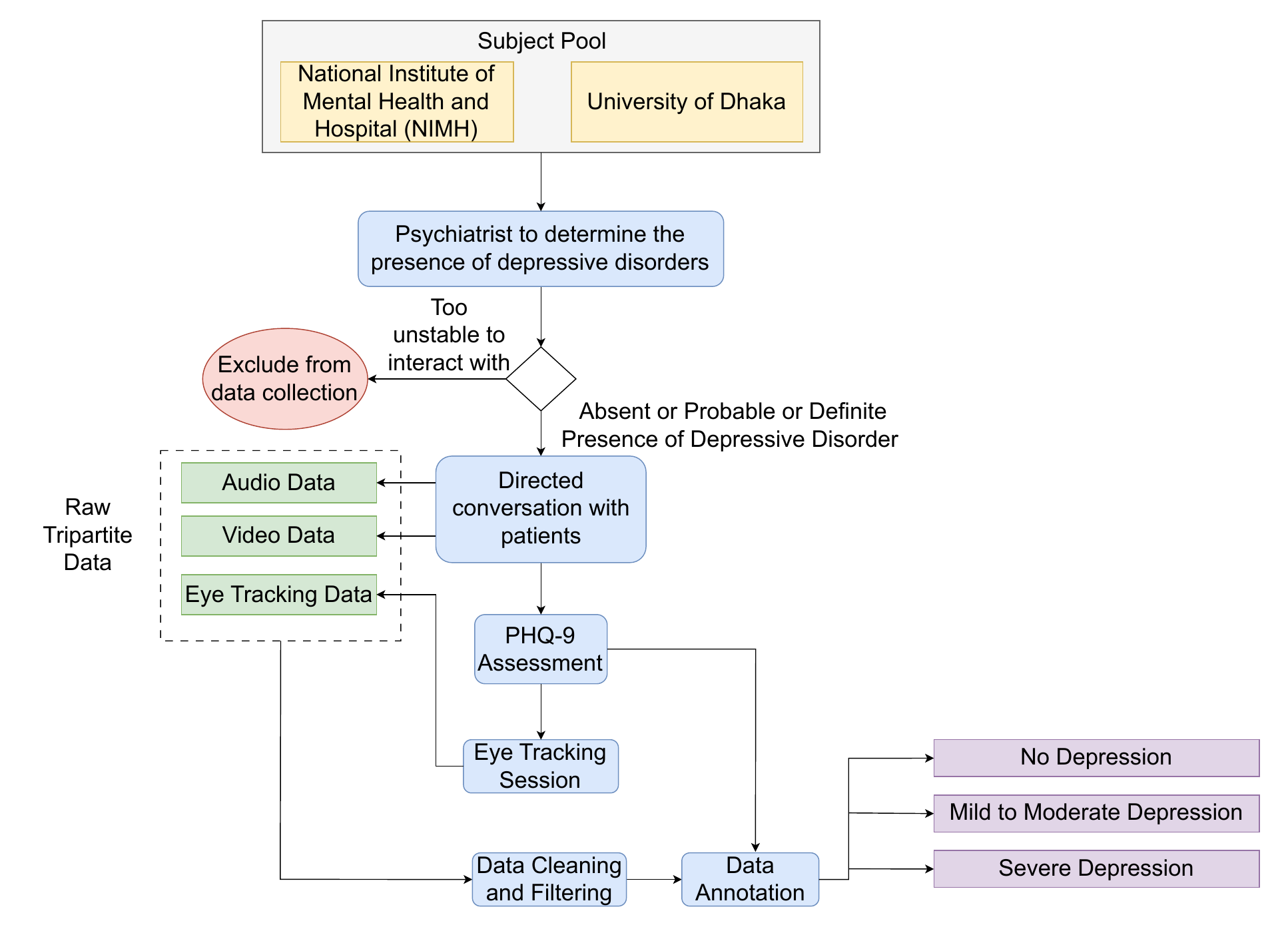}
    \caption{ Illustration of the overall workflow for the data collection conducted at the two collection sites: The National Institute of Mental Health (NIMH) and the University of Dhaka (DU). The data are collected, cleaned for anomalies and missing data, annotated, and separated into the three classes.}
  \label{fig:Tripartite_Data_Collection_Illustration}
\end{figure}

After the initial screening by a licensed psychiatrist, all participants who were deemed composed enough to sit through the data collection session were moved to the next session. With each of the selected participants, we held practice trials before the main study to ensure the participants were familiar with the experimental procedures and that the data were collected correctly.


Existing studies show that those suffering from MDD are more likely to focus on negative emotions than positive emotions and have a greater negative attention bias for negative stimuli \cite{duque2015double}. Building on this work, we developed a set of 48 images with facial expressions: 12 happy faces, 12 sad faces, and 24 neutral faces. These images were selected from the OASIS image library. All images were processed with Adobe Photoshop software with consistent size, gradation, and resolution (i.e., image size 354 $\times$ 472 pixels; 8 $\times$ 6 cm). The stimulus consisted of pairs of images that included an emotional and a neutral facial expression. Three types of images were shown to the subjects: images that had a neutral face alongside a happy face, images with a neutral face and a sad face, and images with a sad face and a happy face. The order of the images is randomized, but the order in which the images were shown was kept the same during the eye tracking session; this helped to retain a degree of homogeneity in the expected eye movement of participants. Figure \ref{fig:Facial_Expression_Selection} shows how the images were formed, aligning images of different emotions adjacently. 
Figure \ref{fig:data_collection_procedure} shows the setup of the eye tracker with the monitor used for each of the participants.
The positions of the two images within each of the ``combined" images were counterbalanced. Each image with visual stimuli was shown for 3 seconds, followed by an entirely black image (all pixel values set at 0) for 2 seconds. This was done to allow the gaze of the subjects to `reset' such that they could view the succeeding images with visual stimuli without carrying any form of bias from the preceding image. The duration of the task was 100 seconds. Eye-tracking information was collected using a GazePoint GP3 HD Eye Tracker \footnote{https://www.gazept.com} with a sampling frequency of 150 Hz. The experimental stimuli were shown on a 15.6-inch Light Emitting Diode (LED) screen laptop.
Each trial began with a brief calibration session specified by the eye tracker manufacturer. Afterward, the images were shown as described for the durations above while the subject viewed them, and the eye tracking camera collected the data. Inside the eye tracker software, all the measurements were made in milliseconds.A webcam is affixed atop the monitor to record the facial expressions of participants during eye-tracking sessions, as seen in Figure \ref{fig:data_collection_procedure}. Participants are videotaped upon consenting to data collection and positioning before the display. While administering the PHQ-9, participants respond to the questionnaire,  exhibiting an array of facial expressions. These are captured by the webcam and encoded into video data.

\begin{figure}[!h]
    \centering
  \includegraphics[width=0.8\linewidth]{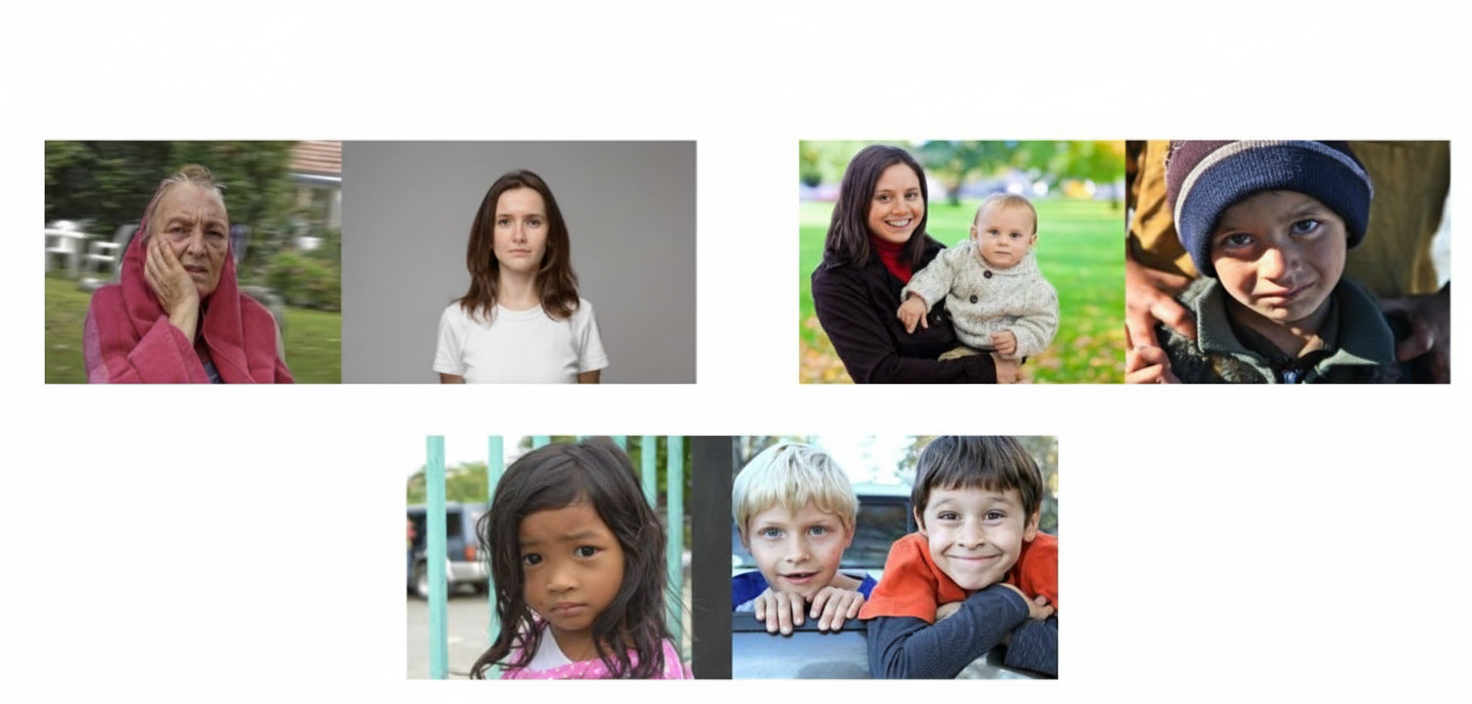}
    \caption{The three combinations of images shown to the subjects: (top-left) Sad-Neutral, (top-right) Neutral-Happy, (bottom) and Happy-Sad.}
  \label{fig:Facial_Expression_Selection}
\end{figure}


Figure \ref{fig:data_collection_procedure} illustrates the overall experimental setup.

\begin{figure}[!h]
    \centering
  \includegraphics[width=0.85\linewidth,trim={0.01cm 0.1cm 0.1cm 0cm},clip]{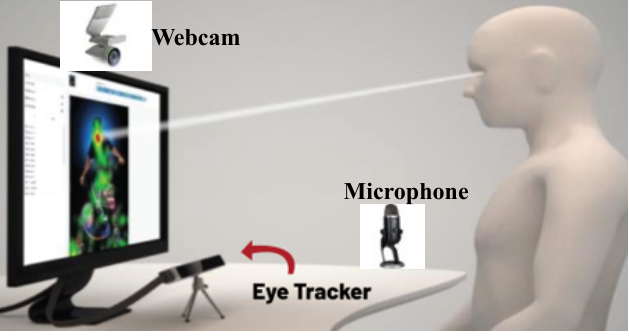}
    \caption{Illustration of the overall experimental setup. As the PHQ-9 questionnaire assessment is administered by the data collection team, the webcam mounted on the monitor records the facial expressions of the subject, while the microphone records the audio from the verbal interaction. After the verbal assessment, the eye tracker, placed below the monitor, is used to track the subject's gaze patterns as they view the selected images.}
  \label{fig:data_collection_procedure}
\end{figure}

\begin{figure}[!h]
    \centering
  \includegraphics[width=\linewidth, trim={0.05cm 0.05cm 0.05cm 0.05cm}, clip]{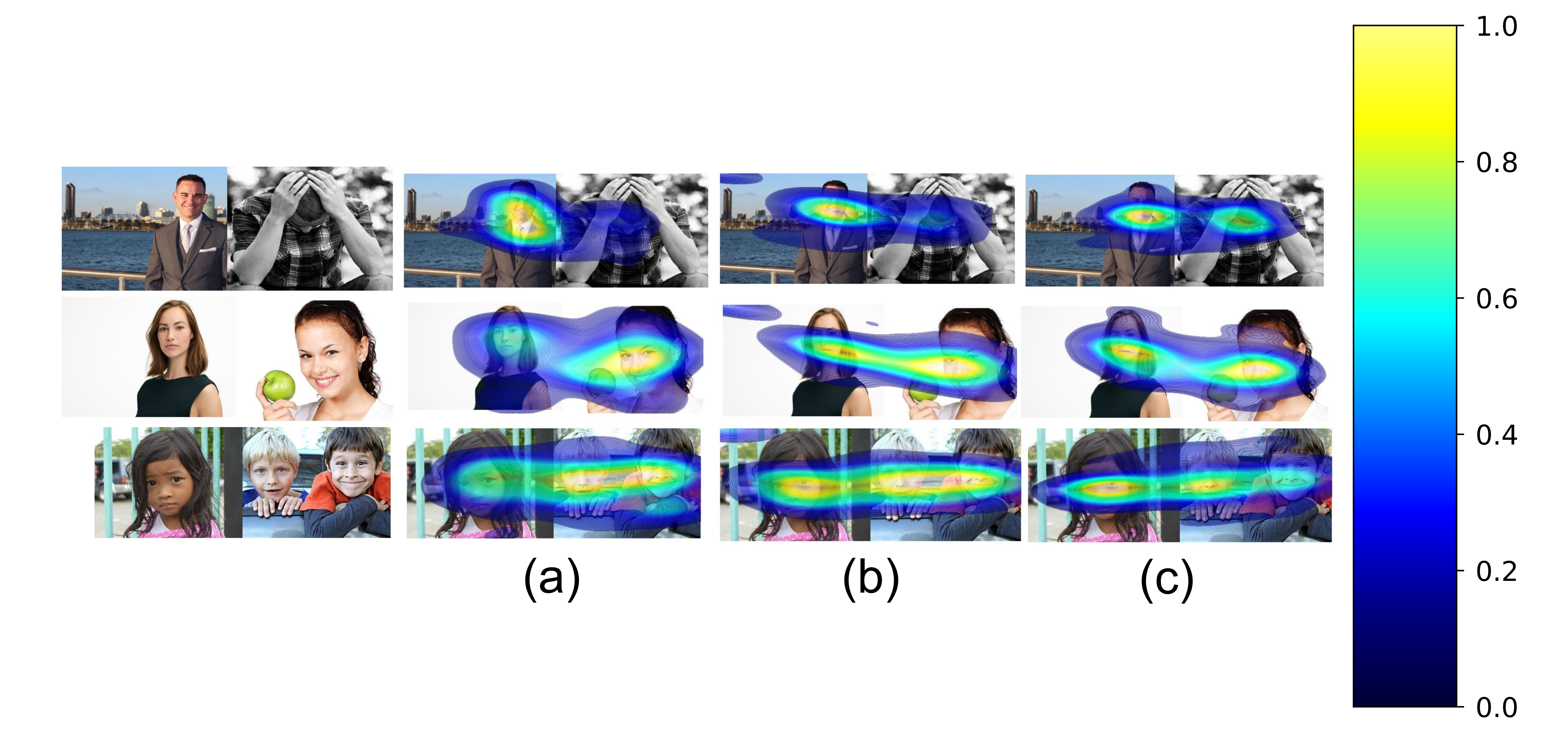}
    \caption{
    %
   Saliency heatmaps generated from eye tracking session data showing attention patterns across depression severity levels. (a) No Depression group, (b) Mild to Moderate Depression group, and (c) Severe Depression group. The heatmaps reveal distinct visual attention patterns, with colour intensity indicating fixation density (yellow = highest attention, blue = lowest attention)."}
    \label{fig:Saliency_Heatmaps}
\end{figure}


Following the eye-tracking step, participants engaged in a directed yet casual conversation with our data collectors. By defining a set of generalized questions, the data collectors attempted to converse with the participants and to delve deeper into the answers that they provided without forcing a conversation. During this time,we recorded their verbal responses, while they conversed, with an audio recorder. This approach facilitates a more comprehensive dataset from which many vocal features can be extracted, enhancing the subsequent analysis by the deep learning model. The extended interaction allows the algorithm to discern a broader range of acoustic patterns, thereby improving the accuracy and reliability of emotion recognition.


The video recording captures facial expressions, body language, and non-verbal cues, while the audio recording serves to capture the loudness, pitch, and other subtle variations in their voices. Following the conversation, the participants completed the PHQ-9 questionnaire in the presence of a licensed psychiatrist.

\section{Proposed Methodology}
\subsection{Preliminaries and Motivation}
In depression classification, several methods has been proposed based on a single modality (unimodal) features, while ignoring the other modalities (e.g., visual maps, facial expression, etc.). Since, multimodal data can provide supplimentary information about the task, these unimodal methods lack the necessary expressive power. Hence, in this work, we focused on fully harnessing the power of multimodal data for depression detection. Majority of existing multimodal approaches consist of two two-stage process. The first stage encompasses feature extraction, where we extract modality-specific features (unimodal features) from our dataset using existing feature extractor modules \cite{rahman2008cognitive,min2023detecting}. For our dataset, the extracted unimodal features are (see figure \ref{fig:proposed_approach}\textit{(Left)}) visual Saliency, facial expression, and acoustic features. These features are initially learned through independent unimodal frameworks that capture the modality-specific features. For the unimodal module, we followed a visual saliency-based framework similar to Rahman \textit{et. al.} \cite{rahman2021classifying} that extracts the gaze pattern related features. The visual and acoustic features are extracted in accordance with Min \textit{et. al.}\cite{min2023detecting}. After extracting these modality specific features there are some works \cite{min2023detecting} that concatenated these features and applied traditional machine learning algorithms, such as decision trees, random forests, SVM on top of this. The formulation of the method can be sumarized as:
\begin{align}
    \hat{Y} = f_{Uni}\bigg(\mathbf{U}_{Audio} \oplus \mathbf{U}_{Gaze} \oplus \mathbf{U}_{Video}\bigg)  \label{eqn:ensemble}
\end{align}
where, $\hat{Y}$ is the predicted outcome, $\mathbf{U}_i$ is the extracted unimodal features for the modality $i$, $\oplus$ represents the concatenation operation and $f_{Uni}$ is the traditional machine learning algorithms (e.g., SVM, KNN etc).

Concretely, in the cross-modality learning, they consider each modality as a node, $v_i \in \mathcal{V}$, where $\mathcal{V}$ is the set of nodes (i.e., in our case $||\mathcal{V}||=3$) equipped with a feature vector, $h(v_i) \in \mathbb{R}^d$ ($d$ is the dimension of the feature), that are learned through a respective unimodal architecture. Since all these modalities are interrelated in a complementary aspect, consequently, they construct a complete graph or dense graph (i.e., all modalities are connected with each other through edges), $\mathcal{G}(\mathcal{V},\mathcal{E})$, where, $\mathcal{E}$ is the edge set. The formulation can be expressed as:
\begin{align}
    \hat{Y} = f_{Cross}\bigg(\mathcal{G}(\mathcal{V}, \mathcal{E}), \mathbf{U}_{Audio}, \mathbf{U}_{Gaze}, \mathbf{U}_{Video} \bigg)  \label{eqn:cross}
\end{align}
where, $f_{Cross}$ can be GCN, Transformer based architectures or any other methods can levearages the cross modality interaction.
Inline with these methods shown in \ref{eqn:cross}, we also utilize the graph structure of underlying multimodal data. Although these methods exploits the interaction between different modalities by aggregating information from neighboring nodes, however, these architectures has a significant limitation: it predominantly relies on low-frequency information, potentially overlooking the nuanced high-frequency signals that can carry critical context for capturing the depression level (See. Lemma \ref{prop:geem_low_pass}). This focus on low-frequency content may lead to suboptimal representation, as high-frequency signals can be indicative of unique or anomalous characteristics that can influence the task. Give an example. To address this limitation, we included a novel graph convolutional module namely ``\textbf{M}ulti-\textbf{F}requency \textbf{F}ilter-\textbf{B}ank \textbf{M}odule" (\textbf{MFFBM}) that can easily be integrated with the existing graph convolutional methods \cite{kipf2017semi}   or existing Transformer architectures \cite{chen_2024_iifdd}\cite{rajan_2022_cross_self}, without requiring any extra learnable parameters. Moreover, this \textbf{MFFBM} module does not include any extra computational burden such as eigendecomposition, higher order power of adjacency matrix. Furthermore, our \textbf{MFFBM} module significantly enhances traditional graph-based frameworks by leveraging both low- and high-frequency information to create a more comprehensive representation for depression classification.

\subsection{Overall Architecture}

Our proposed depression detection framework is composed of two stages: (1) \textit{Unimodal Feature Extraction}, and (2) \textit{Multimodal Feature Extraction via Graph Convolutional Network (GCN)}. As illustrated in Figure \ref{fig:proposed_approach}, the pipeline begins with processing raw input data from three modalities: audio, video, and visual saliency. Each modality is passed through its corresponding feature extractor module that transforms the raw signals into fixed-length embedding vectors of dimension $(n, 64)$, where $n$ is the number of subjects. These embeddings encode local temporal or spatial patterns through convolution, pooling, and Bi-LSTM layers, followed by fully connected layers.In the second stage, the learned unimodal embeddings are used to construct a fully connected graph, where each modality acts as a node. We model inter-modal relationships using a novel multi-frequency GCN block that combines low-pass and high-pass graph filters, designed to capture both commonalities and distinctions across modalities. The resulting cross-modal features are aggregated via global average pooling and concatenated with the original unimodal features to form a comprehensive representation. Finally, this fused representation is passed through a classification head to categorize subjects into three depression severity levels.


\subsection{Unimodal Feature Extractor}

\textit{1)Audio Module:} The audio module uses the Librosa Python library to extract features from audio files. Librosa specialises in music and audio analysis. We extract three features using our audio module - chroma feature, mel spectrograms and Mel-frequency cepstral coefficients(MFCCs). Chroma features can effectively capture pitch variations, intonation patterns, and timbre from different audio files. A Mel spectrogram is a distribution of frequencies in an audio signal in the Mel scale. This is useful for speech processing as it mimics human perception of different frequencies. MFCCs represent the short-term power spectrum of sound, which helps neural networks or machine learning models process human speech effectively. We utilised chroma feature, Mel spectrograms, and MFCCs as input features, as these representations offer a mathematically robust and perceptually meaningful basis for deep learning on speech signals. Each feature aggregates and encodes acoustic information relevant to attributes such as loudness, F0, HNR, jitter, shimmer, F2, Hammaberg Index, and spectral flux, though not always via one-to-one mapping.

\textit{2)Video Module:} Our video module incorporates Face Emotion Recognizer (FER) which is a open source python library to analyze a video frame by frame, detect emotions and categorizing them into seven categories - 'angry', 'disgust','fear','happy', 'sad' and 'neutral' on a scale of 0 to 1. We used FER's default OpenCV Haarcascade classifier to classify the emotions in each video file.

\textit{3)Visual Saliency Module:} For our visual saliency module, we created fixation maps from eye tracking data. Also we used five saliency models to generate saliency maps for the images we showed to the persons while collecting our data. The five saliency models were SalFBNet, MSI-Net, TranSalNet, Saliency Attentive Model (SAM-ResNet) and Saliency Attentive Model (SAM-VGG). We compared the fixation maps with five different saliency maps for each person to get 8 metrics from this. The metrics were AUC Borji, AUC Judd, CC, KLDiv, NSS, Similarity, AUC Shuffled, and Info Gain \cite{bylinskii2018different}.

For each modality \textit{i}, the embeddings have a shape of (n, $m_i$, $f_i$) where n represents the number of subjects, $m_i$ represents the maximum number of samples for modality \textit{i}, $f_i$ represents the feature size of modality \textit{i}. To transform the variable length outputs from each module to a uniform size, these embeddings were processed through a convolution layer to extract local sequential patterns. It was followed by a max pooling layer to get the most salient information. Then, two Bi-LSTM layers were used to capture bidirectional dependencies and two dense layers to get feature vectors of size (n,64) for each modality. These unimodal features are later used in equation \ref{eqn:final} which we discuss in \ref{subsec:cross_modality_learning}

To prepare the unimodal features for our graph based multimodal feature extractor where each modality will be considered as a node, we reshaped the feature vectors to get multi dimensional representation of the unimodal features. This enabled us to perform channel wise concatenation which preserved the characteristics of each modality. It will help us to learn cross modal information.

\subsection{Multimodal Feature Extractor}

\label{subsec:cross_modality_learning}
\begin{figure*}[!t]
    \centering
    \includegraphics[width=.9\linewidth, trim={0.2cm 0.1cm 0.1cm 0.1cm}, clip]{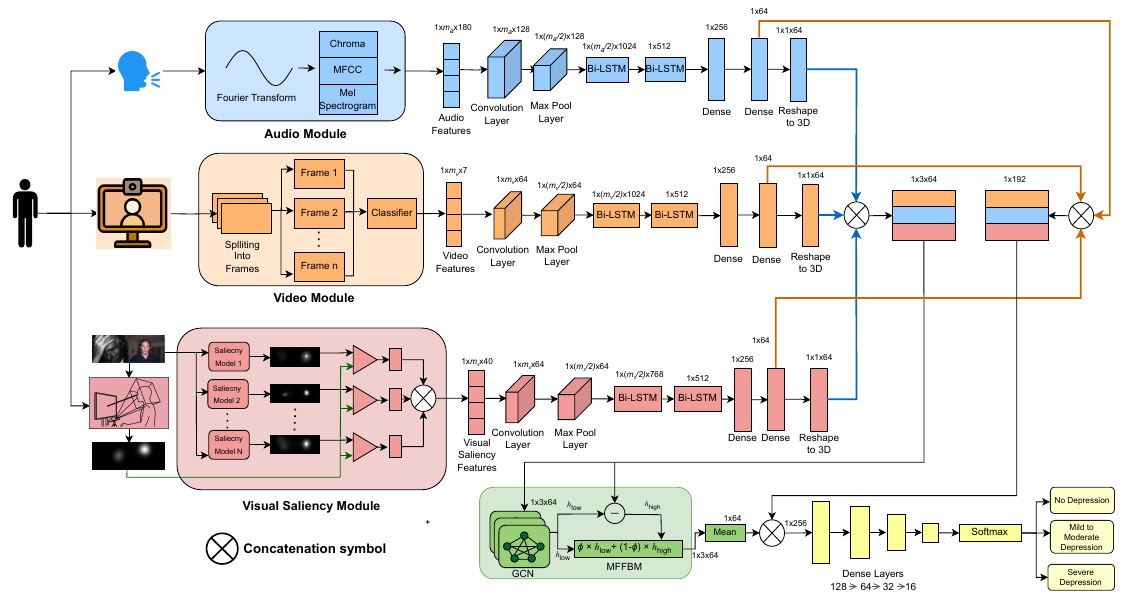}
    \caption{Our proposed multimodal depression detection framework. It starts with the unimodal feature extraction stage. Three parallel branches process audio(highlighted in \textit{blue}), video (\textit{orange}) and visual saliency (\textit{red}) data. Each branch utilizes convolutional and Bi-LSTM layers to produce a 64 64-dimensional embedding vector. In the multimodal feature extraction stage (\textit{green}), a Graph Convolutional Network(GCN) with our Multi frequency filter bank module (MFFBM) models the intermodal relationships. Finally, the resulting cross-modal features are combined with the initial unimodal features and passed to the classification head (\textit{yellow}) to predict the depressed class} 
    \label{fig:proposed_approach}
\end{figure*}

Unlike previous studies \cite{min2023detecting,andersson2021predicting,gao2023abnormal} that treated each audio, video, and gaze modalities as separate entities, these modalities are intrinsically linked through cross-modal relationships, with features that mutually enhance and complement one another. In this work, we exploit this interdependence with Graph Convolutional Neural Network (GNN) as our cross modality learning module. To construct the graph, we assume each modality as a node. Since we are dealing with three distinct modalities, we have 3 separate nodes and each of these nodes are connected with the rest of the nodes through edges (i.e., a dense graph). We denote this graph as $\mathcal{G}(\mathcal{V}, \mathcal{E})$, where $\mathcal{V}$ is the vertex set ($|\mathcal{V}|$ = 3) and $\mathcal{E}$ is the edge set. We assume undirected edges. In GCN models, the convolution operation on the graph is defined as the multiplication of filters and signals in the Fourier domain. Specifically, GCN model learns new node representations by calculating the weighted sum of feature vectors of central nodes and the neighboring nodes. Mathematically, the simplest spectral GCN layer (Kipf and Welling, 2016) can be formulated as:
\begin{align*}
    \mathbf{h}^{\ell+1} = f(\mathbf{h}^{\ell}, \mathcal{A}) = \sigma\bigg(\mathbf{\tilde{\mathcal{A}}} {h}^{\ell} \Theta^{\ell}\bigg)
\end{align*}
where $\mathbf{h}^{\ell}$ is the matrix of activations in the $\ell$-th layer, and $\Theta^{\ell}$ is a layer-specific trainable weight matrix. In addition, $\mathcal{\tilde{A}} = \mathcal{D}^{-\frac{1}{2}} \mathcal{A} \mathcal{D}^{-\frac{1}{2}}$ is the normalized adjacency matrix with self loops, and $\sigma(\cdot)$ is an activation function, such as the $\text{ReLU}(\cdot) = \text{max}(0, \cdot)$. In addition, $\mathcal{D}$ is the diagonal degree matrix, with the $i$-th diagonal element defined as $d_i = \sum_{i \neq j} \mathcal{A}_{ij}$. In the GCN layer in our proposed framework, instead of using a single filter $\Theta$ to extract features from the input feature vector $\mathcal{X}$, we stack multiple of these graph convolutional filters with the adjacency matrix $\mathcal{\tilde{A}}$. This  is computed as:
\begin{align}
    \mathbf{h}^{\ell+1}_{low} = f(\mathcal{A}, \mathcal{X}) = \sum_{i}^k \phi_i\text{ReLU}\bigg((\mathcal{\tilde{A}} \odot \mathcal{M}^{\ell}) {h}^{\ell} \Theta^{\ell}_i\bigg)  \label{eqn:low_pass}
\end{align}
where, $\sum_i \phi_i = 1$ and each $\phi_i$ can be a learnable parameter or for simplicity we may manually set it [0,1] and treat it as a hyperparameter. In this work we choice the second option and treat it as a hyperparameter. We will explore the possibility of learning them in our future study.
It is worth noting that Eqn \ref{eqn:low_pass} primarily act as a low pass operation is and designed to be a function of $\mathcal{\tilde{A}}$, $\mathcal{K}^{(\ell+1)}_{low} =  \mathcal{D}^{-\frac{1}{2}} (\mathcal{A} + \mathcal{I}_{\mathbb{N}}) \mathcal{D}^{-\frac{1}{2}}$, where $\mathcal{D} \in \mathbb{R}^{N \times N}$ is the diagonal degree matrix, $\mathcal{A}$ is the graph adjacency matrix without self-loop, $\mathcal{I}_{\mathbb{N}}$ is the identity matrix and $\mathcal{M}^{(\ell)}$ is the attentin mask at $\ell$-th layer. Intuitively, the eqn \ref{eqn:low_pass} represents a low-pass linear filter by aggregating the node’s self-information with its neighborhood information. Since different frequency components contributes differently in the signal that efffect the overall depression score. Thus, it is reasonable to construct GNNs beyond only low-pass filters, that can extract various spectrum of information through multi-channel filter banks. For this purpose, we designed a simple yet effective high-pass GNN as:
\begin{align}
 \mathbf{h}^{(\ell+1)}_{high} =  \mathbf{a} (\mathcal{\tilde{A} } \odot \mathcal{M}^{(\ell)}) \mathbf{h}^{\ell} \Theta^{\ell} - (1-\mathbf{a}) \mathbf{h}^{\ell} \Theta^{\ell} \label{eqn:high_gcn}   
\end{align}
where, we designed the high pass filter as, $\mathcal{K}^{(\ell+1)}_{high} = -\mathbf{a}\mathcal{\tilde{A}} + (1-\mathbf{a})\mathcal{I}_{\mathbb{N}}$. From eqn \ref{eqn:high_gcn}, our high-pass GNN computes the difference between the self-information and neighborhood information. It highlights the features of a node that are distinct from its neighbors. Finally, to construct the multi-frequency filter banks, we combine these low and high frequency information, $\mathbf{h}_{low}^{(\ell+1)}$ and $\mathbf{h}_{high}^{(\ell+1)}$, resulting in the following equation for our proposed GCN block:
\begin{align}
    \mathbf{h}^{(\ell+1)} = \phi \mathbf{h}_{low}^{(\ell+1)} + (1-\phi) \mathbf{h}_{high}^{(\ell+1)} \label{eqn:overall_eqn}
\end{align}
where, $\phi_j$ 
is the balancing factor since it controls the frequency spectrum of the overall filter. $\phi_j$ set as a hyperparameter and  $\mathbf{h}^{(\ell+1)}$ is the overall feature representation with multi-channel filterbank at $\ell+1$ layer. By stacking multiple GCN blocks we can learn higher-order node features from neighboring nodes. In addition, GCN propagates information on a graph structure and gradually aggregates the information of neighboring nodes, which allows us to effectively capture the complex dependencies in the graph structure. Finally, we employ a simple channelwise global average pooling (GA) to generate graphlevel feature representations and concatenate this with modality specific features (i.e., extracted through unimodal feature extractor). This can be expresed as:

\begin{align}
    \mathcal{Z} = \mathbf{h}^{(L+1)} \oplus \mathcal{U}_{A} \oplus \mathcal{U}_{V} \oplus \mathcal{U}_{G} \label{eqn:final} 
\end{align}
Finally, to enhance the final representation we integrate the cross modality based features, $\mathbf{h}^{(L+1)}$, with the unimodal features, $\mathcal{U}_{A}, \mathcal{U}_{V}, \mathcal{U}_{G}$, through a channel-wise concatenation operation, as shown in equation \ref{eqn:final}.

The final representation of our multimodal feature extractor from equation \ref{eqn:final} was passed through multiple dense layers followed by a softmax layer to classify subjects into one of three depression levels (e.g., No Depression, Mild to moderate depression, Severe depression).

\subsection{Theoretical Analysis of MFFB}
\begin{lemma}
The spectral behavior of a graph convolution kernel $\mathbf{C}$ can be described by its frequency response in the spectral domain, given by:
\begin{align*}
    \mathcal{F}(\lambda) = \operatorname{diag}^{-1}\left(\mathbf{U} \mathbf{C} \mathbf{U}^T\right),
\end{align*}
where $\mathcal{F}(\lambda)$ represents the frequency response, $\mathbf{U}$ is the matrix of eigenvectors of the normalized graph Laplacian $\mathcal{L}$, and $\mathbf{C}$ is the graph convolution operator.
\end{lemma}

This definition follows the spectral interpretation introduced in~\cite{balcilar2020bridging}.


\begin{lemma}
\label{prop:geem_low_pass}
The frequency profile of the graph convolution layer utilized in \textbf{GCN}\cite{kipf2017semi} is defined by:
\begin{align*}
    \mathcal{F}(\lambda) = 1 - \frac{p}{p+1} \lambda  
\end{align*}
where $\lambda$ is the eigenvalues of the normalized graph Laplacian and the given graph is an undirected regular graph whose node degree is $p$.
\end{lemma}
The filter response linearly decreases until the third quarter of the spectrum (cut-off frequency) where it reaches zero. From this point, it increases linearly until the spectrum's end. We don't include the proof here due to space constraints, but readers can find a detailed proof of Lemma \ref{prop:geem_low_pass} in Balcilar et al. \cite{balcilar2020bridging}.

\begin{theorem}
Our Proposed parametrization GCN in Eqn \ref{eqn:overall_eqn} can equivalently express arbitrary graph filter with continuous frequency response function.
\end{theorem}
\begin{proof}
As shown in \ref{eqn:overall_eqn}, our multifrequency GCN architecture has two components. Initially, we capture low frequency signals ($\mathbf{h}^{(\ell+1)}_{low}$) similar to GCN \cite{kipf2017semi}. Subsequently, we identify high frequency components ($\mathbf{h}^{(\ell+1)}_{high}$) by computing the difference between the original input features and the extracted low frequency representation. These complementary signal components are then integrated using the weighting parameter $\beta_i$. For a GCN with $L$ layers, Equation~\ref{eqn:overall_eqn} can be reformulated as:
\begin{align*}
    \mathbf{h}^{(L)} = \bigg[\phi \bigg(\mathbf{h}^{(1)}_{high}+ \mathbf{h}^{(2)}_{high}+\ldots
    +\mathbf{h}^{(\ell)}_{high}+ \ldots + \\ \mathbf{h}^{(L)}_{high}\bigg) + (1-\phi) \bigg(\mathbf{h}^{(1)}_{low}+\ldots +\mathbf{h}^{(\ell)}_{low}+\ldots + \mathbf{h}^{(L)}_{low}\bigg) \bigg]
\end{align*}
For simplicity we consider the second layer representation for our Eqn $\ref{eqn:overall_eqn}$ as $\mathbf{h}^{(2)} =  \phi_{i} \mathbf{h}_{high}^{(2)} + (1-\phi_{i}) \mathbf{h}_{low}^{(2)}]$. We can express the low frequency components as: $\mathbf{h_{low}}^{(2)} = \mathcal{\tilde{A}}\mathbf{h}^{(1)}\Theta^{2}$, which can be further decomposed to:
\begin{align*}
    \mathbf{h}^{(2)}_{low} &= \mathcal{\tilde{A}}\bigg[\phi \bigg(\mathcal{\tilde{A}}\mathbf{h}^{(0)}\Theta^{(1)}\bigg) + (1-\phi) \bigg(\mathcal{\tilde{A}}\mathbf{h}^{(0)}\Theta^{(1)} \\ &\qquad \qquad \qquad \qquad \qquad \qquad- \mathbf{h}^{(0)}\Theta^{(1)}\bigg)\bigg] \Theta^2 \\
    &= \bigg[\Theta \bigg(\mathcal{\tilde{A}}^2 \mathbf{h}^{(0)} \Theta^{(1)} \bigg) + (1-\phi) 
    \\ & \qquad\qquad\quad\quad\quad \ \ \ \ \bigg(\mathcal{\tilde{A}}^2 \mathbf{h}^{(0)}\Theta^{(1)}  - \mathcal{\tilde{A}}\mathbf{h}^{(0)} \bigg) \bigg] \Theta^{(2)} \\
    &= \mathcal{\tilde{A}}^2 \mathbf{h}^{(0)}\Theta^{(1)}\Theta^{(2)} - (1-\phi) \bigg(\mathcal{\tilde{A}} \mathbf{h}^{(0)} \Theta^{(1)} \Theta^{(2)}\bigg) \\
    &= \bigg(\mathcal{\tilde{A}}^2- (1-\beta) \mathcal{\tilde{A}}\bigg) \mathbf{h}^{(0)} \Theta^{(1)} \Theta^{(2)}
\end{align*}
where, $\mathcal{K}_1 = (\mathcal{\tilde{A}}^2- (1-\phi) \mathcal{\tilde{A}})$ is the corresponding convolutional kernel and $\mathbf{h}^{(0)}$ is the input feature matrix. We can also express this convolutional kernel with the help of graph filter response function $\mathcal{F}(\cdot)$ as: $\mathcal{K}_1 = \mathbf{U} diag(\mathcal{F}_1(\lambda)) \mathbf{U}^T$ \cite{balcilar2020bridging}. Hence, we analysis the frequency profile of the $\mathbf{h}^{(2)}_{low}$ of the convolutional kernel $\mathcal{K}_1$ as: $\mathcal{F}_1(\lambda) = diag^{-1} (\mathbf{U}^T \mathcal{K}_1 \mathbf{U})$. Replacing the normalized graph adjacecncy matrix $\mathcal{\tilde{A}}$ with the graph Laplacian $\mathcal{L}$ in the convolutional kernel $\mathcal{K}_1$ as $\mathcal{K}_1 = (\mathcal{I} - \mathcal{L})^2 - (1-\phi_i) (\mathcal{I} - \mathcal{L})$. We express the graph laplacian $\mathcal{L}$ with the help of Eigendecomposition as $\mathcal{L} = \mathbf{U} diag(\mathbf{\lambda}) \mathbf{U}^T$, where $\mathbf{\lambda}$ is the eigenvalues of the graph laplacian and $\mathbf{U}$ is the corresponding eigenvectors. Similarly, we can write $\mathcal{I} - \mathcal{L}$ with Eigendecomposition as: $\mathcal{I} - \mathcal{L} = \mathbf{U} \mathbf{U}^T - \mathbf{U} diag(\lambda)\mathbf{U}^T = \mathbf{U} \mathcal{I} \mathbf{U}^T - \mathbf{U} diag(\lambda)\mathbf{U}^T = \mathbf{U}^T diag(\mathcal{I}-\lambda). \mathbf{U}^T$. Using the rule of power we rewrite $(\mathcal{I}-\mathcal{L})^2$ using the same eigenvalues of $\mathcal{L}$ as $\mathbf{U} diag(\lambda-\mathcal{I})^2 \mathbf{U}^T$. Adding $(\mathcal{I} - \mathcal{L})^2$ and $(\mathcal{I} - \mathcal{L})$ together we express $\mathbf{C}_1$ with respect to the filter response as: $\mathbf{C}_1 = \mathbf{U} diag\bigg((\lambda-1)^{2}- (1-\phi)(1-\lambda)\bigg)\mathbf{U}^T = \mathbf{U} diag\bigg(\lambda^2 - (1+\phi)\lambda - (1-\phi) \bigg) \mathbf{U}^T$, where $\mathcal{F}_1(\lambda) = \lambda^2 - (1+\phi)\lambda - (1-\phi_i)$ is the filter response for the first part, $\mathbf{h}^{(2)}_{low}$, of the our equation \ref{eqn:overall_eqn}. 

Now we expand the second part of our Eqn \ref{eqn:overall_eqn}:
\begin{align*}
    \mathbf{h}^{(2)}_{high} &= \mathcal{\tilde{A}}\bigg[\phi \bigg(\mathcal{\tilde{A}}\mathbf{h}^{(0)}\Theta^{(1)}\bigg) + (1-\phi) \\ & \bigg(\mathcal{\tilde{A}}\mathbf{h}^{(0)}\Theta^{(1)}\bigg)\bigg] \Theta^{(2)}- \bigg[ \phi \bigg(\mathcal{\tilde{A}}\mathbf{h}^{(0)}\Theta^{(1)}\bigg) + \\& \qquad \qquad \qquad \qquad (1-\phi)  \bigg(\mathcal{\tilde{A}}\mathbf{h}^{(0)}\Theta^{(1)}\bigg) \bigg] \Theta^{(2)} \\
    &= \mathcal{\tilde{A}}^2 \mathbf{h}^{(0)} \Theta^{(1)}\Theta^{(2)} - \phi \bigg(\mathcal{\tilde{A}}\mathbf{h}^{(0)} \Theta^{(1)}\Theta^{(2)} \bigg) 
    \\ &\qquad \qquad \qquad \qquad \qquad - (1-\phi) \bigg(\mathbf{h}^{(0)} \Theta^{(1)} \Theta^{(2)} \bigg) \\
    &= \bigg( \mathcal{\tilde{A}}^2 - \phi \mathcal{\tilde{A}} + (1-\phi)\mathcal{I}  \bigg) \mathbf{h}^{(0)} \Theta^{(1)} \Theta^{(2)}
\end{align*}
where, $\mathcal{K}_2 = ( \mathcal{\tilde{A}}^2 - \phi \mathcal{\tilde{A}} + (1-\phi)\mathcal{I})$ is the corresponding convolutional kernel for $\mathbf{h}^{(2)}_{high}$. As previously shown we can express the filter response of the convolutional kernel $\mathcal{K}_2$ as $\mathcal{F}_2(\lambda) = \lambda^2 - \lambda + 2\phi $. Finally, we combine the filter response of part $1$ and part $2$ to get the overall frequency profile for our proposed GCN described in Eqn \ref{eqn:overall_eqn} as: $\mathcal{F}(\lambda) = 2\lambda^2 - \phi\lambda + (1+\phi)$ for the second layer, which is a quardadic bandpass filter and the cut off frequency, $\lambda = \frac{\phi \pm \sqrt{\phi^2 - 8(1+\phi)}}{4}$ is controlled using the parameter $\phi$. Thus, our \textbf{MFFBM} in Eqn \ref{eqn:overall_eqn} could express filters with diverse frequency properties, e.g., low/band/high-pass filters.
\end{proof}

\section{Evaluation Metrics}


In this section, we describe the evaluation metrics used to assess the performance of our model. We focus on a comprehensive set of metrics that capture different aspects of model performance, ensuring a balanced evaluation. 



\noindent{\textbf{Precision (PRE.)}}: Precision, also known as Positive Predictive Value, indicates the proportion of true positive predictions among all predictions of the positive class. It is defined as:

\begin{align*}
\text{Precision} = \frac{\text{True Positives}} {\text{True Positives} + \text{False Positives}} 
\end{align*}

Precision is important in contexts where the cost of false positives is high, such as in spam detection. A high precision suggests that the model has a low rate of false positives.

\noindent{\textbf{Recall (REC.)}}: Recall, also known as Sensitivity or True Positive Rate, measures the proportion of true positive predictions among all actual positive instances. It is calculated as:
\begin{align*}
\text{Recall} = \frac{\text{True Positives}} {\text{True Positives} + \text{False Negatives}}    
\end{align*}
Recall is crucial in situations where failing to identify positive instances can have severe consequences, such as in medical testing or fraud detection. A high recall indicates that the model has a low rate of false negatives.

\noindent{\textbf{Specificity (SPEC.)}}: Specificity, also known as True Negative Rate (TNR), calculates the ratio of true negative predictions among all actual negative instances. It is calculated as:
\begin{align*}
    \text{Specificity} = \frac{\text{True Negatives}}{\text{True Negatives} + \text{False Positives}}
\end{align*}
Specificity is particularly important in scenarios where false positive predictions are costly or harmful such as in diagnostics screening where a false alarm can lead to unnecessary treatments or distress. A high specificity means that the model can successfully classify negative classes with minimum number of false positives.



\noindent{\textbf{F2 score:}}The F2 score is a metric used to evaluate the performance of a classification model, particularly in scenarios where the focus is on maximizing recall (identifying as many positive cases as possible) while still considering precision. It is a variation of the F1 score, which balances precision and recall equally. The F2 score, however, gives more weight to recall, making it useful in situations where false negatives (missed positive cases) are more costly than false positives (incorrectly identified positives).
\begin{align*}
\text{F2 Score} = \frac{(1 + 2^2) \cdot \text{Precision} \cdot \text{Recall}}{2^2 \cdot \text{Precision} + \text{Recall}} = \frac{5 \cdot \text{Precision} \cdot \text{Recall}}{4 \cdot \text{Precision} + \text{Recall}} 
\end{align*}

\noindent{\textbf{AUC (Area Under the ROC Curve)}: The \textbf{AUC} measures the area under the Receiver Operating Characteristic (\textbf{ROC}) curve, which plots the True Positive Rate (Recall) against the False Positive Rate at various thresholds. The \textbf{AUC} represents the model's ability to discriminate between positive and negative classes. An \textbf{AUC} of 0.5 indicates a model with no discriminatory power (random guessing), while an \textbf{AUC} of 1.0 represents perfect discrimination. \textbf{AUC} is useful for evaluating a model's performance across different classification thresholds}


Due to the class imbalance in our dataset, we have chosen these metrics to evaluate our model's performance. Furthermore, for the task of depression detection, minimizing false negatives is a higher priority than minimizing false positives. Therefore, F2 score was chosen over F1 score to prioritize recall in our evalution.

\section{Methods For Comparison}


To evaluate the performance of our proposed method, we conducted a comparative analysis against several established machine learning algorithms and deep learning methods. The deep learning methods specified here are used for multimodal data and depression detection. Our comparison focuses on the following algorithms and deep learning methods: Decision Tree (\textbf{DT}), Random Forest (\textbf{RF}), Support Vector Machine (\textbf{SVM}), eXtreme Gradient Boosting (\textbf{XGBoost}), k-Nearest Neighbors (\textbf{KNN}), Naive Bayes (\textbf{NB}) , Logistic Regression (\textbf{LR}), Intra and Inter-modal Fusion for Depression Detection(\textbf{IIFDD}), Self-Attention Transformer(\textbf{TF(S)}) and Cross-Attention Transformer(\textbf{TF(C)}) . Each algorithm and method has unique characteristics that influence its performance, interpretability, and computational requirements. Below, we provide a brief description of each algorithm and method and the rationale for its inclusion in our study.

\noindent{\textbf{Decision Tree (DT)}}: Decision Trees are a basic yet powerful type of algorithm that models data by recursively splitting it based on features that maximize information gain. They are intuitive and interpretable, making them suitable for exploratory data analysis. However, Decision Trees can be prone to overfitting, especially with deep trees.\cite{breiman1984decision}

\noindent{\textbf{Random Forest (RF)}}: Random Forest is an ensemble method that combines multiple Decision Trees through bagging (bootstrap aggregating). It introduces randomness by randomly selecting subsets of features for each tree, reducing overfitting and improving generalization. Random Forest is often more robust and accurate than individual Decision Trees, making it a common benchmark in machine learning studies.\cite{breiman2001random}

\noindent{\textbf{Support Vector Machine (SVM)}}: \textbf{SVM} is a supervised learning algorithm that aims to find the optimal hyperplane that separates different classes. It works well in high-dimensional spaces and is effective for both linear and non-linear classification (with kernel methods). \textbf{SVMs} are robust to overfitting, especially with appropriate regularization, and are commonly used for their strong theoretical foundations.\cite{cortes1995}

\noindent{\textbf{eXtreme Gradient Boosting (XGBoost)}: \textbf{XGBoost} is a gradient boosting algorithm that builds models in a sequential manner, where each subsequent model aims to correct errors from the previous one. It is known for its efficiency, scalability, and ability to handle complex datasets with high dimensionality. \textbf{XGBoost} often achieves high accuracy and is favored in competitive machine learning challenges.\cite{chen2016xg}}

\noindent{\textbf{k-Nearest Neighbors (KNN)}}: \textbf{KNN} is a non-parametric method that classifies samples based on the labels of their nearest neighbors in the feature space. The simplicity of \textbf{KNN} makes it a useful baseline for comparison. However, it can be sensitive to noisy data and requires careful tuning of the 'k' parameter for optimal performance \cite{cover1967}.


\noindent{\textbf{Gaussian Naive Bayes (NB)}}: This is a probabilistic classifier designed for continuous features. It models by fitting a Gaussian distribution to the data for each class. It is computationally simple and perform well when its underlying assumptions are met. Its main limitation it he normality assumption which may not accurately represent the true distribution of the data \cite{george2013gaussian}.

\noindent{\textbf{Logistic Regression (LR)}}: Logistic regression is a statistical model that classifies data by modeling the probability of an outcome using the sigmoid function. Although it is inherently a binary classifier, it can be used for multiclass classification tasks through strategies such as "One vs Rest"(OvR). Due to its computational efficiency and straightforward nature, it is often used as a baseline for classification problems. However, it assumes that input variables have a linear relationship with the log odds of the outcome \cite{mccullagh1989}.

\noindent{\textbf{Intra and Inter-modal Fusion for Depression Detection(IIFDD)}: IIFDD is a transformer based deep learning method that uses inter and intra modal fusion to classify depressed patients using multi modal data - audio, video, and text \cite{chen_2024_iifdd}.}

\noindent{\textbf{Self-Attention Transformer (TF(S))}: This model uses modality-specific encoders and self-attention module to encode features. The concatenated features are then passed to statistical pooling layer to get mean and standard deviation vectors. Then fully connected layers are used to do the final classification \cite{rajan_2022_cross_self}.}

\noindent{\textbf{Cross-Attention Transformer (TF(C))}: This model is similar to the Self-Attention Transformer. But the only difference is that it uses cross multi-head attention instead of self-attention \cite{rajan_2022_cross_self}.}

By comparing these algorithms against our proposed method, we aim to demonstrate its effectiveness and robustness. The comparative analysis provides insights into each method's strengths and weaknesses and helps identify scenarios where our proposed method offers significant advantages.

\section{Overall Result}

\begin{table*}[ht]
\centering
\small
\setlength{\tabcolsep}{4pt}
\begin{tabular}{l*{16}{c}}
\toprule
\multirow{2}{*}{\textbf{Model}} & \multicolumn{4}{c}{\textbf{Ensemble}} & \multicolumn{4}{c}{\textbf{Audio}} & \multicolumn{4}{c}{\textbf{Video}} & \multicolumn{4}{c}{\textbf{Gaze}} \\
\cmidrule(lr){2-5} \cmidrule(lr){6-9} \cmidrule(lr){10-13} \cmidrule(lr){14-17}
& \textbf{Spec} & \textbf{Sens} & \textbf{F2} & \textbf{Pre} & \textbf{Spec} & \textbf{Sens} & \textbf{F2} & \textbf{Pre} & \textbf{Spec} & \textbf{Sens} & \textbf{F2} & \textbf{Pre} & \textbf{Spec} & \textbf{Sens} & \textbf{F2} & \textbf{Pre} \\
\midrule
DT\cite{breiman1984decision}   & 0.70 & 0.71 & 0.72 & 0.74 & \textbf{0.78} & 0.76 & 0.77 & \textbf{0.84} & 0.65 & 0.65 & 0.66 & 0.72 & \textbf{0.70} & 0.59 & 0.61 & 0.73 \\
RF\cite{breiman2001random}   & 0.80 & 0.79 & 0.80 & 0.83 & 0.70 & 0.81 & 0.80 & 0.76 & 0.70 & 0.73 & 0.73 & 0.75 & 0.58 & 0.70 & 0.70 & 0.69 \\
XGB\cite{chen2016xg}  & 0.75 & 0.79 & 0.79 & 0.79 & 0.68 & 0.66 & 0.66 & 0.65 & 0.68 & 0.67 & 0.68 & 0.72 & 0.50 & 0.68 & 0.67 & 0.62 \\
LR\cite{mccullagh1989}   & 0.80 & 0.81 & 0.81 & 0.83 & 0.70 & 0.78 & 0.77 & 0.75 & 0.70 & 0.80 & 0.79 & 0.77 & 0.50 & 0.84 & 0.81 & 0.70 \\
KNN\cite{cover1967}  & 0.78 & 0.85 & 0.84 & 0.82 & 0.70 & 0.79 & 0.78 & 0.76 & 0.65 & 0.74 & 0.73 & 0.71 & 0.63 & 0.62 & 0.63 & 0.67 \\
GNB\cite{george2013gaussian}  & 0.83 & 0.66 & 0.69 & 0.83 & 0.63 & 0.62 & 0.63 & 0.68 & 0.75 & 0.78 & 0.78 & 0.80 & 0.65 & 0.60 & 0.61 & 0.64 \\
SVM\cite{cortes1995}  & \textbf{0.85} & 0.79 & 0.80 & 0.86 & 0.68 & 0.81 & 0.80 & 0.75 & 0.68 & 0.84 & 0.83 & 0.77 & 0.50 & 0.77 & 0.75 & 0.67 \\
\addlinespace
TF(S)\cite{rajan_2022_cross_self} & 0.83 & 0.89 & 0.89 & 0.87 & 0.65 & 0.98 & 0.93 & 0.78 & 0.73 & 0.91 & 0.89 & 0.81 & 0.60 & 0.86 & 0.81 & 0.67 \\
TF(C)\cite{rajan_2022_cross_self} & 0.80 & 0.93 & 0.91 & 0.85 & 0.63 & \textbf{1.00} & \textbf{0.94} & 0.77 & 0.68 & \textbf{0.94} & \textbf{0.91} & 0.80 & 0.38 & \textbf{0.96} & \textbf{0.88} & 0.67 \\
\addlinespace
IIFDD\cite{chen_2024_iifdd} & \textbf{0.85} & 0.93 & 0.92 & \textbf{0.90} & 0.65 & 0.98 & 0.93 & 0.78 & 0.85 & 0.89 & 0.89 & \textbf{0.88} & 0.55 & 0.90 & 0.86 & 0.74 \\
MF-GCN  & 0.83 & \textbf{0.96} & \textbf{0.94} & 0.88 & 0.73 & 0.81 & 0.81 & 0.80 & \textbf{0.88} & 0.80 & 0.81 & \textbf{0.88} & 0.68 & 0.78 & 0.78 & \textbf{0.80} \\
\bottomrule
\end{tabular}
\caption{Comparison of Model Performance Across Modalities For Binary Classification (Subject-wise split using 10-Fold Cross-Validation)}
\label{tab:binary-class-10}
\end{table*}

\begin{table*}[ht]
\centering
\small
\setlength{\tabcolsep}{4pt}
\begin{tabular}{l*{16}{c}}
\toprule
\multirow{2}{*}{\textbf{Model}} & \multicolumn{4}{c}{\textbf{Ensemble}} & \multicolumn{4}{c}{\textbf{Audio}} & \multicolumn{4}{c}{\textbf{Video}} & \multicolumn{4}{c}{\textbf{Gaze}} \\
\cmidrule(lr){2-5} \cmidrule(lr){6-9} \cmidrule(lr){10-13} \cmidrule(lr){14-17}
& \textbf{Spec} & \textbf{Sens} & \textbf{F2} & \textbf{Pre} & \textbf{Spec} & \textbf{Sens} & \textbf{F2} & \textbf{Pre} & \textbf{Spec} & \textbf{Sens} & \textbf{F2} & \textbf{Pre} & \textbf{Spec} & \textbf{Sens} & \textbf{F2} & \textbf{Pre} \\
\midrule
DT\cite{breiman1984decision}   & 0.70 & 0.53 & 0.53 & 0.54 & 0.77 & 0.63 & 0.64 & 0.69 & 0.70 & 0.52 & 0.52 & 0.52 & 0.72 & 0.52 & 0.53 & 0.55 \\
RF\cite{breiman2001random}   & 0.77 & 0.67 & 0.67 & 0.69 & 0.77 & 0.63 & 0.64 & 0.69 & 0.75 & 0.63 & 0.63 & 0.64 & \textbf{0.76} & 0.58 & 0.59 & \textbf{0.64} \\
XGB\cite{chen2016xg}  & 0.73 & 0.61 & 0.61 & 0.59 & 0.77 & 0.65 & 0.65 & 0.67 & 0.72 & 0.58 & 0.58 & 0.59 & 0.69 & 0.45 & 0.46 & 0.48 \\
LR\cite{mccullagh1989}   & 0.80 & 0.71 & 0.72 & 0.74 & 0.80 & 0.68 & 0.69 & 0.72 & 0.77 & 0.66 & 0.67 & 0.70 & 0.73 & 0.57 & 0.58 & 0.63 \\
KNN\cite{cover1967}  & 0.81 & 0.73 & 0.73 & 0.75 & 0.79 & 0.67 & 0.68 & 0.71 & 0.75 & 0.60 & 0.61 & 0.66 & \textbf{0.76} & 0.58 & 0.59 & \textbf{0.64} \\
GNB\cite{george2013gaussian}  & 0.75 & 0.64 & 0.64 & 0.64 & 0.76 & 0.64 & 0.64 & 0.63 & 0.76 & 0.56 & 0.57 & 0.60 & 0.68 & 0.44 & 0.45 & 0.49 \\
SVM\cite{cortes1995}  & 0.79 & 0.69 & 0.70 & 0.74 & 0.80 & 0.68 & 0.69 & 0.71 & 0.75 & 0.63 & 0.63 & 0.61 & 0.74 & 0.57 & 0.58 & 0.63 \\
\addlinespace
TF(S)\cite{rajan_2022_cross_self} & 0.83 & 0.75 & 0.74 & 0.69 & 0.80 & 0.70 & 0.70 & 0.69 & 0.77 & 0.70 & 0.70 & 0.68 & 0.69 & 0.58 & 0.58 & 0.56 \\
TF(C)\cite{rajan_2022_cross_self} & 0.82 & 0.77 & 0.76 & 0.71 & 0.81 & \textbf{0.73} & 0.73 & 0.74 & 0.80 & \textbf{0.74} & 0.73 & 0.70 & 0.66 & 0.54 & 0.52 & 0.44 \\
\addlinespace
IIFDD\cite{chen_2024_iifdd} & 0.84 & 0.77 & 0.77 & 0.79 & 0.81 & 0.71 & 0.72 & 0.78 & 0.80 & 0.71 & 0.72 & 0.74 & 0.75 & \textbf{0.62} & \textbf{0.62} & 0.62 \\
MF-GCN  & \textbf{0.87} & \textbf{0.79} & \textbf{0.79} & \textbf{0.81} & \textbf{0.82} & \textbf{0.73} & \textbf{0.74} & \textbf{0.79} & \textbf{0.85} & \textbf{0.74} & \textbf{0.74} & \textbf{0.74} & 0.75 & 0.60 & 0.60 & 0.59 \\
\bottomrule
\end{tabular}

\caption{Comparison of Model Performance Across Modalities For 3 class classification (Subject-wise split using 10-Fold Cross-Validation)}

\caption*{\textbf{Model Abbreviations:} 
DT = Decision Tree, RF = Random Forest, XGB = XGBoost, LR = Logistic Regression,
KNN = k-Nearest Neighbors, GNB = Gaussian Naive Bayes, SVM = Support Vector Machine, 
GCN = Graph Convolutional Network, TF(C) = Cross-Attention Transformer, 
TF(S) = Self-Attention Transformer \\[8pt]
\textbf{Metric Abbreviations:} 
Sens = Sensitivity, Spec = Specificity, F2 = F2-Score, Pre = Precision} 
\label{tab:multi-class-10}
\end{table*}

\begin{table*}[ht]
\centering
\small
\setlength{\tabcolsep}{4pt}  
\begin{tabular}{l*{12}{c}}
\toprule
\multirow{2}{*}{\textbf{Model}} & \multicolumn{3}{c}{\textbf{Ensemble}} & \multicolumn{3}{c}{\textbf{Audio}} & \multicolumn{3}{c}{\textbf{Text}} & \multicolumn{3}{c}{\textbf{Video}} \\
\cmidrule(lr){2-4} \cmidrule(lr){5-7} \cmidrule(lr){8-10} \cmidrule(lr){11-13}
& \textbf{Recall} & \textbf{F2} & \textbf{Pre} & \textbf{Recall} & \textbf{F2} & \textbf{Pre} & \textbf{Recall} & \textbf{F2} & \textbf{Pre} & \textbf{Recall} & \textbf{F2} & \textbf{Pre} \\
\midrule
DT\cite{breiman1984decision} & 0.90 & 0.90 & 0.91 & 0.89 & 0.90 & 0.97 & 0.59 & 0.62 & 0.80 & 0.85 & 0.85 & 0.87 \\
RF\cite{breiman2001random} & 0.90 & 0.91 & 0.95 & 0.92 & 0.93 & 0.97 & 0.69 & 0.72 & 0.84 & 0.85 & 0.86 & 0.90 \\
XGB\cite{chen2016xg} & 0.85 & 0.84 & 0.80 & 0.85 & 0.87 & 0.96 & 0.66 & 0.68 & 0.80 & 0.85 & 0.84 & 0.80 \\
LR\cite{mccullagh1989} & 0.95 & 0.96 & \textbf{1.00} & 0.92 & 0.93 & 0.97 & 0.77 & 0.78 & 0.80 & 0.85 & 0.85 & 0.87 \\
KNN\cite{cover1967} & 0.95 & 0.96 & \textbf{1.00} & 0.92 & 0.93 & 0.97 & 0.77 & 0.77 & 0.78 & 0.85 & 0.85 & 0.87 \\
GNB\cite{george2013gaussian} & 0.90 & 0.92 & \textbf{1.00} & 0.89 & 0.90 & 0.97 & 0.65 & 0.66 & 0.72 & 0.85 & 0.85 & 0.87 \\
SVM\cite{cortes1995} & 0.95 & 0.96 & \textbf{1.00} & 0.92 & 0.93 & 0.97 & 0.73 & 0.75 & 0.84 & 0.85 & 0.85 & 0.87 \\
\addlinespace
TF(S)\cite{rajan_2022_cross_self} & 0.90 & 0.92 & \textbf{1.00} & 0.92 & 0.93 & 0.97 & 0.80 & 0.83 & 0.95 & \textbf{0.90} & \textbf{0.90} & 0.90 \\
TF(C)\cite{rajan_2022_cross_self} & 0.95 & 0.96 & \textbf{1.00} & \textbf{0.96} & \textbf{0.96} & 0.97 & 0.77 & 0.79 & 0.91 & 0.85 & 0.85 & 0.87 \\
\addlinespace
IIFDD\cite{chen_2024_iifdd} & 0.95 & 0.95 & 0.95 & \textbf{0.96} & \textbf{0.96} & 0.97 & 0.72 & 0.76 & 0.97 & 0.80 & 0.83 & \textbf{1.00} \\
MF-GCN & \textbf{0.95} & \textbf{0.96} & \textbf{1.00} & 0.92 & 0.93 & 0.97 & \textbf{0.77} & \textbf{0.81} & \textbf{1.00} & 0.85 & 0.86 & 0.90 \\
\bottomrule
\end{tabular}
\caption{Comparison of Model Performance Across Modalities For Binary classification (Using CMDC Dataset) }
\label{tab:cmdc}
\end{table*}


Experimental Setup: We optimize the proposed multimodal framework via the Adam \cite{KingBa15} optimizer, with the learning rate of $0.001$, training epoch of $500$ accompanying an early stopping of 50 patience and minibatch size of 16. The proposed model is implemented based on Tensorflow 2.0 \cite{sanchez2020review}, and the model is trained by using a single GPU (NVIDIA RTX 3090 with 24 GB memory). All the hyperparameters in Eqn \ref{eqn:overall_eqn} are empirically set using a grid search \cite{syarif2016svm}. To evaluate our model, we used 10-fold cross-validation. We ensured that in each fold, train and test sets have mutually exclusive sets so that any data leakage doesn't happen. 10-fold cross-validation helps us to give a more comprehensive view of our model's generalizability than the single train-test split method. We formulate the depression detection problem as a ``Multiclass Classification Task". Due to the limited number of collected samples currently we divide the entire dataset into only three classes, namely \textbf{Class 1}, \textbf{Class 2} and \textbf{Class 3}, where each class has there respective PHQ-9\footnote{PHQ-9 score ranges between 0 to 27, the larger the score the higher is the depression level.} score range. Concretely, the no depression cases (i.e., PHQ-9: 0-4) fall in the \textbf{Class 1}, the mild and moderate depression cases (i.e., PHQ-9: 5-14) are in \textbf{Class 2}, and severe depression patients (i.e., PHQ-9: 15-27) are in \textbf{Class 3}. This classification range is further discussed and verified by the professional doctors in our team. As the evaluation metrics we are using are primarily used for binary classification, we used the weighted average method to calculate them for multi-class classification.

The quantitative results of the proposed framework and ten competing methods in the task of depression classification are reported in Table \ref{tab:binary-class-10} to Table \ref{tab:multi-class-10}. In each table, we use the term feature ensemble to concatenate the uni modal features i.e., video, audio, and gaze. This feature concatenation or ensemble does not consider the dependence among different modalities and treat features from each modality independently. In contrast, our proposed method (\textbf{MF-GCN}, bottom row of every table) exploit the correlation among various modalities (e.g., video, audio and gaze) through our cross-modality learning module with the help of graph convolution layer and make predictions based on both the unimodal features and cross-modality features. This helps our proposed multimodal method to significantly outperform the seven machine learning algorithms that solely rely on the ensemble features or concatenated features from different modalitie,s and also beats the deep learning methods that uses cross modality. 


Our cross-modality learning methods (i.e., \textbf{MF-GCN}, \textbf{IIFDD}, \textbf{TF(S)} and \textbf{TF(C)}) generally achieve better performance in terms of five metrics, compared with seven traditional machine learning methods.

In Table \ref{tab:binary-class-10}, binary classification results using 10 fold cross validation were reported. For single modality features, our model might not perform best across all metrics but gives consistent performance. For video features, our model achieves the highest specificity and precision. For Gaze features, we excel at precision. However, for ensemble features, we can see that our model outperforms all other methods in Sensitivity and F2 Score with a score of $0.96$ and $0.94$ respectively. This highlights our model's ability to correctly identify positive cases with minimum false positives. For binary classification, the second best model can be considered IIFDD as it achieves the second best Sensitivity and F2-Score.

Multiclass classification results using 10 fold cross validation were reported in Table \ref{tab:multi-class-10}. The results from this table resonate with our findings from the Table \ref{tab:binary-class-10}. However, in this case, our model outperforms every method across all modalities except at gaze features, where IIFDD performs better in sensitivity and F2 Score. When all modalities are considered, we can see that we gain significant performance improvement over any other methods. We achieve a Sensitivity score of $0.79$ and Specificity score of $0.87$, where the second best model \textbf{IIFDD} achieves respectively $0.77$ and $0.84$.


These results demonstrate that our proposed multimodal methods can learn diagnostic oriented salient cross-modality features more effectively, and also highlight the importance of these cross- modality features in depression detection. This is the key reason behind the performance boost for all five metrics for our method compared with traditional machine learning methods that rely solely on ensemble features and other deep learning methods. 

Our proposed multimodal framework outperforms the feature ensemble based existing baselines but it also generally outperform over all the unimodal features, especially in case of multiclass classification. This can be attributed to the MFFBM module. The module captures both inter class boundaries and distinctive high frequency patterns. This is reciprocated in Figure \ref{fig:roc_comparison} where we can see that class 1 which is more challenging to predict as it has overlapping features, our model still perform the best. For the more simpler task of binary classification, the performance gap between our model and others is less pronounced but our model still achieved highest sensitivity which is the most critical metric in clinical context.
Also, we observe that the feature ensemble based approach generally outperforms their single-modality counterparts (i.e., individual video, audio or gaze). For instance, both DT and XGBoost methods that ensemble unimodal features (e.g., video, gaze and audio) are superior to the unimodal DT and XGBoost (see coloum namely \textbf{Video}, \textbf{Audio}, \textbf{Gaze} of Table \ref{tab:multi-class-10}) which only use modality specific data. This implies that taking advantage of multimodal data (as we do in this work) helps promote the diagnosis performance.


To find out, our model does not exhibit bias to the data we collected, we evaluated the performance of every method using the CMDC\cite{zou_2022_CMDC_dataset} dataset.


CMDC: The CMDC dataset contains semi-structured interviews with MDD patients from China and provides audio, visual and textual features derived from those interviews. During the interviews, they were asked 12 questions. The questionnaire responses are transcribed and included in the dataset. 78 samples (26 depressed and 52 healthy) are presented in this dataset, where every subject was audio recorded, and only 45 (19 depressed and 26 healthy) were both audio and video recorded. We used 5 fold cross validation to evaluate this model.

The number of samples in the CMDC dataset is quite small and depressed samples are much less than the healthy ones. To address this, we employ a data augmentation process\cite{chen_2024_iifdd}. Given that each sample's data consists of a sequence of 12 responses, we generate new training samples by randomly shuffling the order of these responses. This process is first applied 12 times to all samples which is followed by a second application only on depressed samples. Depressed samples are also augmented 12 times.



CMDC dataset results are given in Table \ref{tab:cmdc}. Our model demonstrates superior and consistent performance across all modalities. It achieves the highest performance along with other models with a Recall of 0.95, F2-Score of 0.96 and Precision of 1.00. For single modalities, our model shows competitive performance with the highest results in text with Recall and F2 of 0.77 and 0.81, respectively while maintaining solid performance in audio and video.

The lower performance when using only textual features can be attributed to the loss of non verbal information when audio was transcribed to text. Unlike the involuntary and non verbal cues that acoustic or visual features have, words are consciously articulated and may be semantically guarded or ambiguous.

Therefore, these results imply the superiority of our proposed multimodal framework to help to boost the overall performance of depression identification.


\begin{figure*}[htbp] 
    \centering
    \begin{subfigure}[b]{0.4\linewidth}
        \centering
        \includegraphics[width=\linewidth, trim={0.2cm 0.2cm 0.2cm 0.2cm}, clip]{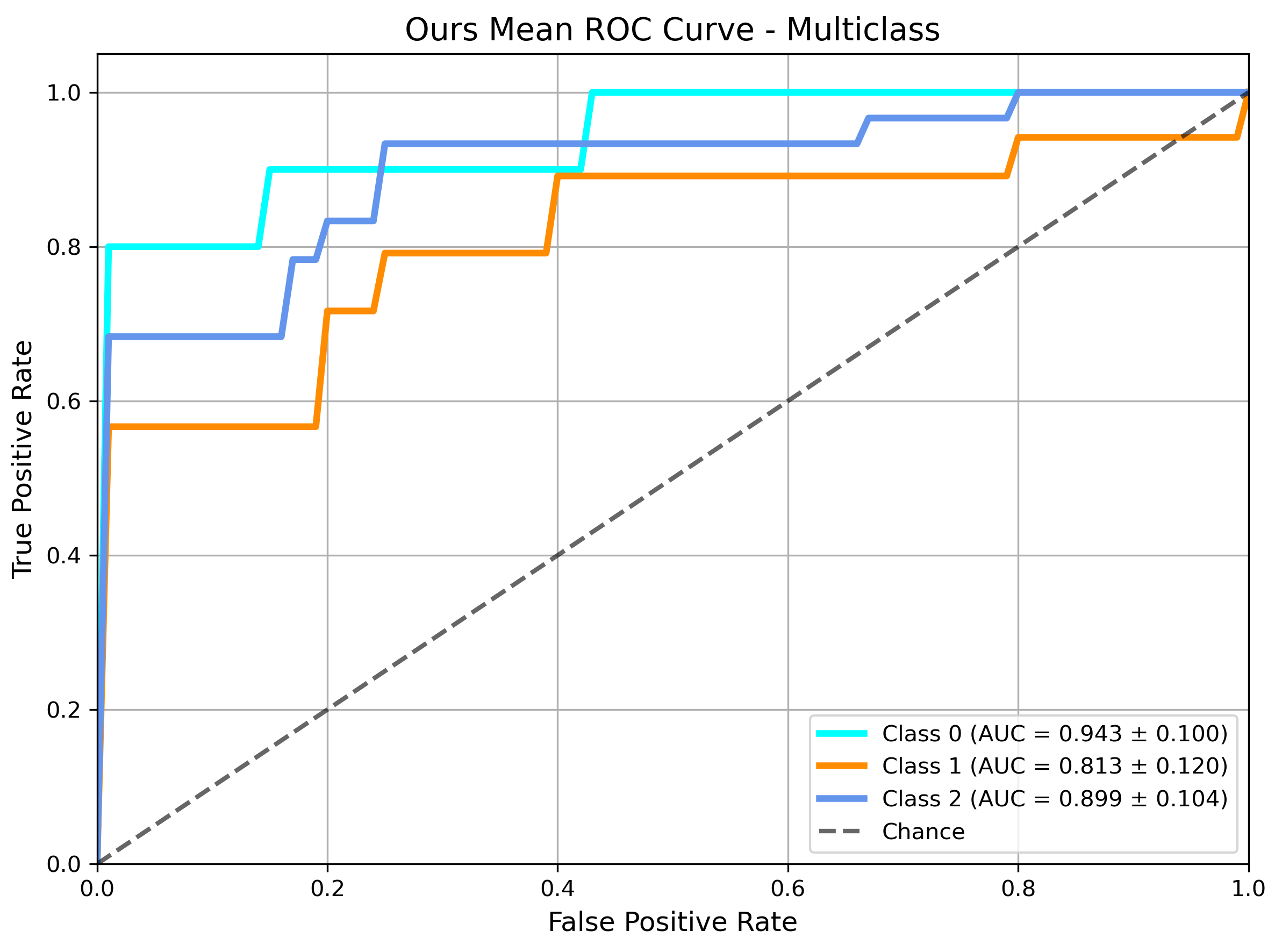}
        \caption{ROC Curve analysis of our proposed multimodal model}
        \label{fig:roc_proposed}
    \end{subfigure}%
    \hspace{0.01\linewidth}
    \begin{subfigure}[b]{0.4\linewidth}
        \centering
        \includegraphics[width=\linewidth, trim={0.1cm 0.2cm 0.2cm 0.2cm}, clip]{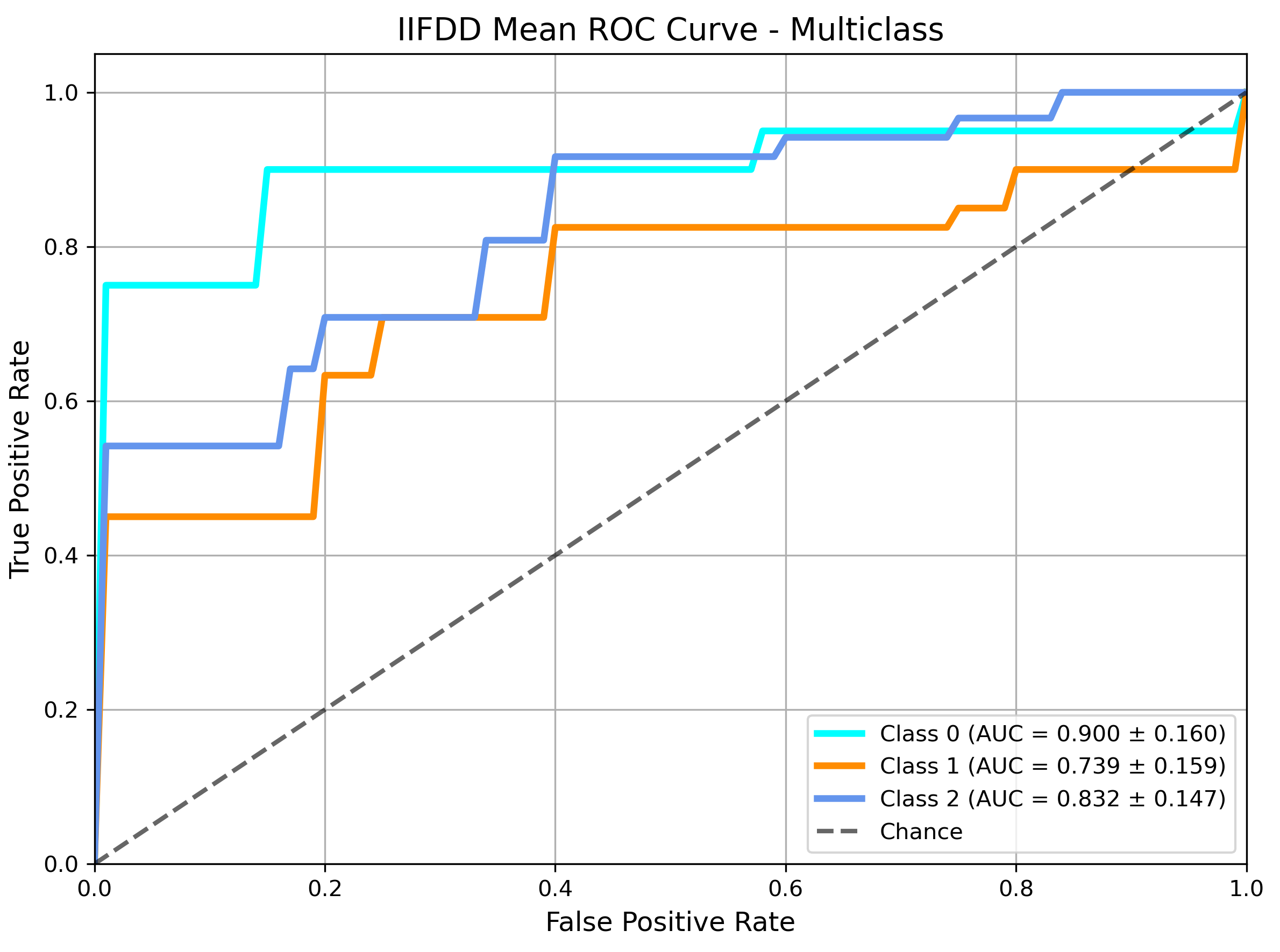}
        \caption{ROC Curve analysis of IIFDD}
        \label{fig:roc_baseline}
    \end{subfigure}
    \caption{Comparison of ROC curves between (a) our proposed multimodal approach and (b) second best competing method i.e., IIFDD, and related area values of individual classes, the higher the area the better the result}
    \label{fig:roc_comparison}
\end{figure*}

\section{ROC Curve Analysis}

ROC Curve plot was updated. This plot is based on multiclass classification


In this section, we analyze the Receiver Operating Characteristic (\textbf{ROC}) curves for our proposed method (\textbf{MF-GCN}) in comparison with the second best method \textbf{IIFDD}, from the Table \ref{tab:multi-class-10}. Although \textbf{ROC} curves are typically designed for binary classification tasks, beyond the binary classification in this work we are performing multiclass depression classification, hence we extended their application to a multiclass setting by adopting a \textit{``one-vs-all"} approach. This method involves creating a binary \textbf{ROC} curve for each class (see Figure \ref{fig:roc_comparison}, treating it as a positive class while grouping all other classes as the negative class. To generate \textbf{ROC} curves for our multiclass task, we calculated the true positive rate (\textbf{TPR}) and false positive rate (\textbf{FPR}) for varying threshold levels across each class. Using the \textit{``one-vs-all"} strategy, we treated each class individually against all other classes, effectively transforming the multiclass problem into multiple binary classification tasks. With these \textbf{TPR} and \textbf{FPR} values, we plotted the \textbf{ROC} curve for each class and calculated the area under the curve (\textbf{AUC}). The \textbf{AUC} values were then used to assess and compare the performance of the two methods—our proposed method (\textbf{MF-GCN}) and \textbf{IIFDD}. The \textbf{AUC} values derived from the \textbf{ROC} curves provide a quantitative measure of each algorithm's ability to distinguish between classes. Higher \textbf{AUC} values indicate better discrimination power, with a value of $1$ signifying perfect classification and a value of $0.5$ representing a random guess.

Regarding the ROC curves for our proposed method and IIFDD in Figure \ref{fig:roc_comparison}(a) \& (b), a key observation is our model's exceptional performance for Class 0 which is the minority class. Our model achieved a remarkable AUC of $0.943$ against IIFDD's $0.9$. This indicated our model can effectively capture discriminative features even from a small sample. The most challenging case was class 1 for both models. We got an AUC score of $0.81$ while IIFDD scored only $0.74$. Although class 1 is the majority class, the relatively lower AUC scores indicate that this class has overlap of features from both class 0 \& class 2 and the discriminative features of class 1 is not that pronounced. However, our model still showed high performance in this challenging case while IIFDD struggled. We can see similar significant margin for class 2 as well where our method and IIFDD got AUC of $0.9$ and $0.83$ respectively. We can conclude from this curve that our method has a robust generalization capability and provides stability which can be seen by the superior AUC scores in each class regardless of the class imbalance.

\section{Ablation Study}

\begin{figure}[htbp]
    \centering
    \includegraphics[width=1\linewidth, trim={0.2cm 0.2cm 0.2cm 0.2cm}, clip]{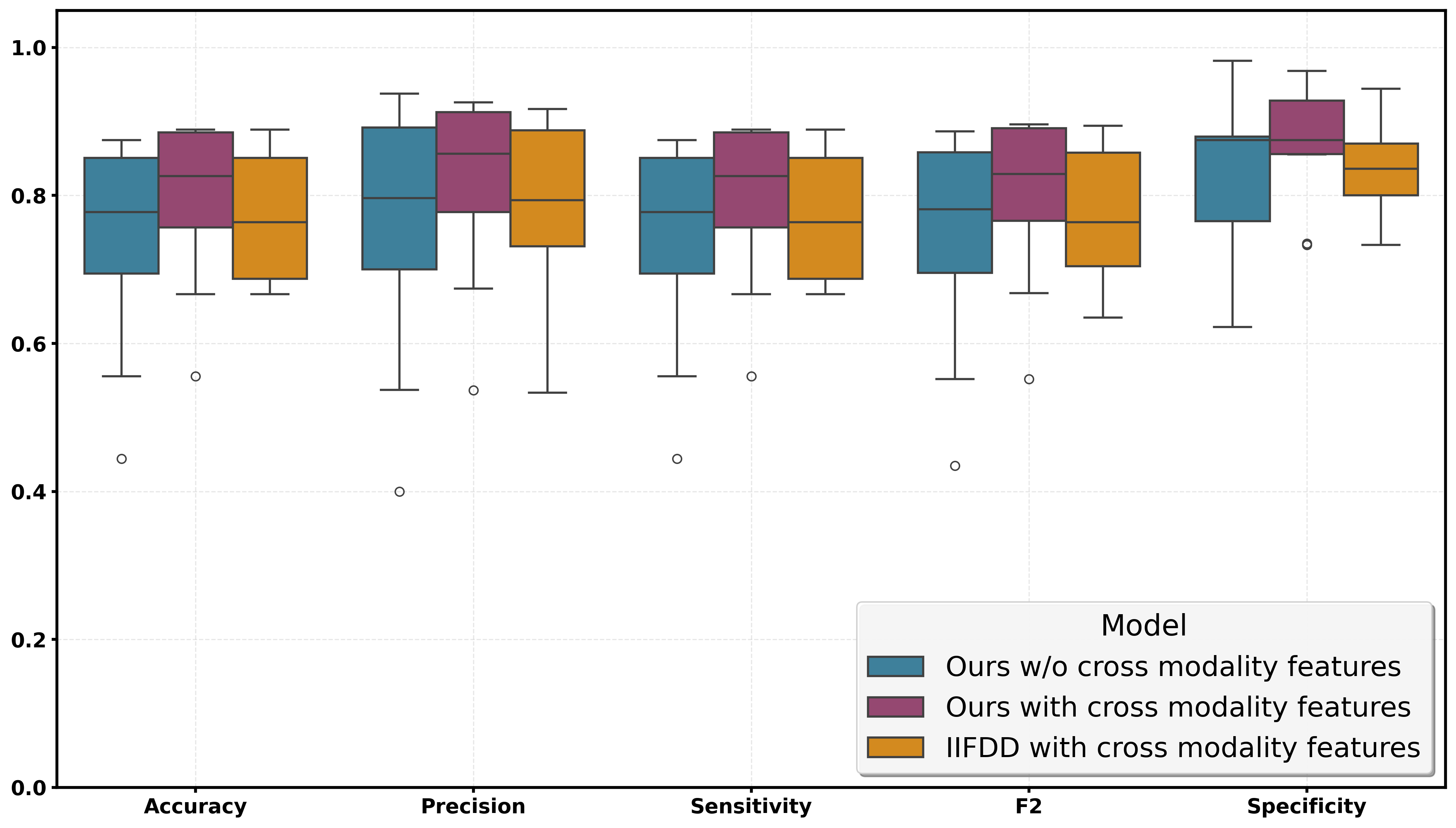}
    \caption{ Analyze the effect of our cross modality features with the help of five evaluation metrics}
    \label{fig:ablation}
\end{figure}

Updated the ablation study figure \ref{fig:ablation} and writing accordingly

In this ablation study, we aim to evaluate the impact of the cross modality learning module on the performance of our model. The cross modality learning module is designed to capture interactions between distinct modalities, leveraging the complementary information they provide. To validate its effectiveness, we conducted a comprehensive ablation study, comparing the performance of our model with and without this module (see Figure \ref{fig:ablation}, Ours cross modality features and Ours w/o cross modality features respectively). Additionally, we included the performance metrics for the second best model \textbf{IIFDD} from the Table \ref{tab:multi-class-10} to provide a comparative benchmark. In  detail, from \ref{fig:ablation}, we first removed the cross modality learning module from our proposed architecture, namely ours w/o cross modality features which is similar to only using a deep neural netweork, allowing us to observe the performance when the model relies solely on modality-specific concatenated features. We then computed the performance metrics, including accuracy, precision, recall, specificity and F2 score, to evaluate the impact of this removal. We ran the same performance metrics on \textbf{IIFDD}, which served as a baseline. This allowed us to compare the effectiveness of the reduced model against a standard benchmark.  Finally, we reintroduced the cross modality learning module to our proposed architecture (see Figure \ref{fig:ablation}, Ours with cross modality features) and calculated the same performance metrics to assess any improvements. The results were plotted in a box plot, which provides a comprehensive visualization of the data distribution across different setup. A box plot effectively illustrates the distribution, median, inter-quartile range and potential outliers, offering a concise summary of the central tendency and variability of each metric for each model. 

The ablation study results indicated a significant increase in all performance metrics consistently when the cross modality learning module was included. The increased performance with the cross modality learning module can be attributed to its ability to capture cross-modal interactions and leverage complementary information between distinct modalities. This enhanced capability allows the model to learn a richer representation of the data, leading to more accurate and reliable predictions. If IIFDD is taken into account which also has cross-modality features, we can see that it has a lower median across all metrics. However, it scores higher than our model without cross-modality. Also, we can see a higher variability for IIFDD than our method with cross- modality features across all metrics. This suggests that our method with cross modality features has a better generalization capability than IIFDD. In conclusion, this ablation study effectively demonstrates the impact of the cross modality learning module, empirically validating its effectiveness. By systematically comparing the model's performance with and without this module, we have established that including cross modality interactions contributes to a substantial improvement across key performance.

\section{Conclusion}

This work presents a significant advancement in depression detection through the introduction of a comprehensive trimodal dataset and novel analytical framework. Our gold standard dataset of 103 participants, combining audio, video, and eye-tracking data with psychiatrist-validated PHQ-9 scores, addresses critical gaps in existing depression detection practices. 
Our theoretical and empirical contributions demonstrate that high-frequency spectral information, previously overlooked in graph-based approaches, contains essential diagnostic features for depression detection. The proposed framework achieves 90\% accuracy in depression classification, substantially outperforming existing methods. Furthermore, our systematic evaluation of saliency-based approaches in multimodal contexts yields significant improvement in F1-score compared to unimodal methods, highlighting the complementary nature of different modalities in capturing subtle depression indicators.

The theoretical analysis validating that our module enables arbitrary spectral filter learning represents a fundamental advancement over conventional graph convolutional networks' fixed filtering limitations. This spectral filtering capability, validated across multiple machine learning algorithms and deep learning models, demonstrates statistically significant improvements in depression detection across all severity levels.

These contributions collectively establish a new foundation for multimodal depression detection research, providing both the data resources and analytical tools necessary for more accurate and comprehensive mental health assessment.

\bibliographystyle{naturemag}
\bibliography{refs}


\end{document}